\documentclass[letterpaper]{article} % DO NOT CHANGE THIS
\usepackage{aaai2026}  % DO NOT CHANGE THIS
\usepackage{times}  % DO NOT CHANGE THIS
\usepackage{helvet}  % DO NOT CHANGE THIS
\usepackage{courier}  % DO NOT CHANGE THIS
\usepackage[hyphens]{url}  % DO NOT CHANGE THIS
\usepackage{graphicx} % DO NOT CHANGE THIS
\usepackage{csquotes}
\usepackage{natbib}  % DO NOT CHANGE THIS AND DO NOT ADD ANY OPTIONS TO IT
\usepackage{caption} % DO NOT CHANGE THIS AND DO NOT ADD ANY OPTIONS TO IT
\usepackage{algorithm, algpseudocode}
\usepackage{amsmath, amsthm, amssymb}
\usepackage{booktabs}
\usepackage{multirow} 

\newtheorem{theorem}{Theorem}
\newtheorem{lemma}{Lemma}

\newtheorem{remark}{Remark}

\newtheorem{assumption}{Assumption}

\usepackage{newfloat}
\usepackage{listings}
\DeclareCaptionStyle{ruled}{labelfont=normalfont,labelsep=colon,strut=off} % DO NOT CHANGE THIS
\floatstyle{ruled}
\newfloat{listing}{tb}{lst}{}
\floatname{listing}{Listing}
\title{Split Conformal Prediction in the Function Space with Neural Operators}
\author {
    David Millard\textsuperscript{\rm 1},
    Lars Lindemann\textsuperscript{\rm 2},
    Ali Baheri\textsuperscript{\rm 1}
}
\affiliations {
    \textsuperscript{\rm 1}Department of Mechanical Engineering, Rochester Institute of Technology, Rochester, NY 14623, USA \\
    \textsuperscript{\rm 2}Department of Computer Science, University of Southern California, Los Angeles, CA 90089, USA \\
    djm3622@rit.edu,
    llindema@usc.edu,
    akbeme@rit.edu
}

\begin{document}

\maketitle
\begin{abstract}
\begin{quote}
Uncertainty quantification for neural operators remains an open problem in the infinite-dimensional setting due to the lack of finite-sample coverage guarantees over functional outputs. While conformal prediction offers finite-sample guarantees in finite-dimensional spaces, it does not directly extend to function-valued outputs. Existing approaches (Gaussian processes, Bayesian neural networks, and quantile-based operators) require strong distributional assumptions or yield conservative coverage. This work extends split conformal prediction to function spaces following a two step method. We first establish finite-sample coverage guarantees in a finite-dimensional space using a discretization map in the output function space. Then these guarantees are lifted to the function-space by considering the asymptotic convergence as the discretization is refined. To characterize the effect of resolution, we decompose the conformal radius into discretization, calibration, and misspecification components. This decomposition motivates a regression-based correction to transfer calibration across resolutions. Additionally, we propose two diagnostic metrics (conformal ensemble score and internal agreement) to quantify forecast degradation in autoregressive settings. Empirical results show that our method maintains calibrated coverage with less variation under resolution shifts and achieves better coverage in super-resolution tasks.
\end{quote}
\end{abstract}

\section{Introduction}
Conformal prediction (CP) provides a framework to obtain prediction sets with distribution-free coverage guarantees in finite-sample settings. CP guarantees that ground-truth outcomes lie within predicted sets with a probability $1-\alpha$. While CP has been extensively studied for finite-dimensional prediction tasks involving scalar or vector outputs, comparatively little work has investigated how its finite-sample guarantees extend to function-valued predictions. This gap poses unresolved computational and theoretical challenges \cite{10.3150/21-BEJ1447}. Existing approaches to functional uncertainty quantification (UQ), such as Gaussian processes \cite{SCHULZ20181}, Bayesian neural networks \cite{Goan2020}, or quantile neural operators \cite{ma2024calibrateduncertaintyquantificationoperator}, rely on restrictive distributional assumptions or provide probably approximately correct (PAC) bounds \cite{haussler2018probably}. While PAC-style bounds can offer desirable conditional guarantees, they often require assumptions on model class complexity or generalization behavior. CP, by contrast, guarantees marginal coverage for any predictor.

Addressing the challenges of CP in function spaces is especially relevant for neural operators. This class of models approximates mappings between infinite-dimensional input–output pairs. Architectures such as Fourier neural operators (FNOs) \cite{li2021fourierneuraloperatorparametric} and DeepONets \cite{Lu_2021} offer the potential for discretization-invariant predictions, although real-world performance depends on both the resolution (the number of points sampled) and grid scheme (the method we use to sample those points) \cite{BAHMANI2025118113}. Recent work have extended neural operators to generative settings \cite{KIRANYAZ2021294,Liu_2025_CVPR,NIPS2016_d947bf06}, enabling stochastic sampling of solutions. However, their suitability for CP remains unexplored. More broadly, little prior work has both (i) formally extended CP to the infinite-dimensional function spaces and (ii) studied this in relation to neural operators. 

To bridge these gaps, we develop a conformal prediction framework tailored to function-valued outputs and neural operators. Our approach combines theoretical extensions with practical tools for uncertainty quantification in infinite-dimensional settings. Our main contributions are as follows:
\begin{enumerate}
    \item An extension of split conformal prediction to function spaces. This method reduces the variation, w.r.t the standard $L^2$ norm, of the conformal radius (the $1-\alpha$ quantile on the residual norm) dependent on grid geometry.
    \item A heuristic model of the resolution-dependent conformal radius. This method enables for the improvement in downstream task such as super-resolution.
    \item Two diagnostic metrics, conformal ensemble score and internal agreement. These metrics increase the reliability of ensemble forecasting by building on the components of conformal prediction.
\end{enumerate}

\section{Related Work}

\subsubsection{Uncertainty Quantification in Operator Learning} A key challenge in operator learning is producing reliable uncertainty estimates. Prior methods often relied on distributional assumptions like approximate Gaussian processes~\cite{akhare2023diffhybriduquncertaintyquantificationdifferentiable,zou2023uncertaintyquantificationnoisyinputsoutputs}, pointwise variance estimates~\cite{guo2023ibuqinformationbottleneckbased}, or used loss-based formulations that lacked formal coverage guarantees~\cite{Lara_Benitez_2024}. While a significant advance by Ma et al.~\cite{ma2024calibrateduncertaintyquantificationoperator} provides probably approximately correct (PAC) guarantees for simultaneous, pointwise coverage on a discretized grid, this does not ensure coverage for the function in the continuous domain. Similarly, prior work hierarchically applies split conformal prediction at varying scales and computes the union of their bounds \cite{baheri2025multi}. In contrast, our approach uses split conformal prediction to construct a single uncertainty set over the entire function space. This yields a domain-wide, distribution-free guarantee that is robust to grid discretization.

\subsubsection{Uncertainty Quantification in Ensemble Forecasting} 
Ensemble forecasting is a standard technique for UQ, especially in weather modeling, using initial condition perturbation or deep ensembles~\cite{tran2020novel,scoccimarro1998transients}. This is often an empirical approach, lacking guarantees on the resulting distribution. To approximate the posterior of such models, methods like Monte Carlo (MC) Dropout~\cite{folgoc2021mc} and MC-Sampling~\cite{shapiro2003monte} are commonly used but do not inherently provide a guaranteed coverage probability.
The quality of these probabilistic forecasts is typically evaluated using metrics like the continuous ranked probability score (CRPS)~\cite{Pic_2023,bulte2025uncertainty}. However, a good CRPS does not translate to a formal guarantee for any single forecast. While post-hoc techniques like temperature scaling can improve upon this lack of guarantees~\cite{kull2019beyond}, the literature has not sufficiently investigated forecasting quality via  CP~\cite{qian2023uncertainty,durasov2021masksembles,rahaman2021uncertainty}. To this end, we integrate ensemble methods with conformal calibration to create a diagnostic tool that indicates when an autoregressive ensemble forecast has degraded beyond a predefined reliability threshold of $1-\alpha$.

\section{Preliminaries and Problem Formulation}

\subsubsection{Conformal Prediction} CP provides a distribution-free framework for constructing prediction sets with finite-sample coverage guarantees. Given i.i.d. training data \(\{(x_i, y_i)\}_{i=1}^n \sim \mathcal{D}\)  and a predictive model \(\hat{f}\colon \mathcal{X} \to \mathcal{Y}\), the goal is to construct a prediction set \(\Gamma_\alpha(x)\) such that:
\begin{equation}
\mathbb{P}_{\{(x_i, y_i)\}_{i=1}^n,\, (x, y) \sim \mathcal{D}} \left[ y \in \Gamma_\alpha(x) \right] \geq 1 - \alpha.
\end{equation}
Split CP begins by partitioning the data into two disjoint sets: one for model training \(\mathcal{D}_{\text{train}}\) and another for calibration \(\mathcal{D}_{\text{cal}}\). A model \(\hat{f}\) is trained on \(\mathcal{D}_{\text{train}}\), and nonconformity scores are computed on \(\mathcal{D}_{\text{cal}} = \{(x_i, y_i)\}_{i=1}^m\) via $s_i = \mathcal{A}(x_i, y_i) = \| \hat{f}(x_i) - y_i \|$.
The quantile \(\tau_\alpha\) is then defined as the  \(\lceil(1-\alpha)(n+1)\rceil\)-th value of the ordered scores from the calibration set, and the prediction set is given by:
\begin{equation}
\Gamma_\alpha(x) = \left\{ y : s(\hat{f}(x), y) \leq \tau_\alpha \right\},
\end{equation}
where \(s(\hat{f}(x), y)\) denotes the nonconformity score between the prediction \(\hat{f}(x)\) and candidate output \(y\). In our case, this is the relative weighted \(L^2\) error. This procedure guarantees that for a new test point, the nonconformity score \(s(\hat{f}(x), y)\) will not exceed this threshold with probability at least \(1-\alpha\): \(\mathbb{P}_{\{(x_i, y_i)\}_{i=1}^n,\, (x, y) \sim \mathcal{D}}(s(\hat{f}(x), y) \le \tau_\alpha) \ge 1-\alpha\).

\subsubsection{Neural Operators} 
Neural operators generalize classical neural networks to learn mappings between infinite-dimensional function spaces. Formally, given a (possibly nonlinear) operator \(\mathcal{G} \colon \mathcal{X} \to \mathcal{Y}\) acting on Banach spaces \(\mathcal{X}, \mathcal{Y}\), the goal is to learn a parametric approximation \(\mathcal{G}_\theta\) such that
$
\mathcal{G}_\theta : f \mapsto u \approx \mathcal{G}(f),
$
where \(f \in \mathcal{X}\) is typically a coefficient or source term in a PDE, and \(u \in \mathcal{Y}\) is the corresponding solution. Unlike standard deep learning models, which learn pointwise mappings, operator learning learns entire function-to-function maps, enabling inference on new input functions and resolutions. A generic neural operator maps \(f \mapsto u\) via a composition of a lifting map, multiple integral transformations, and a final projection back to the input domain. We adopt the standard Fourier Neural Operator (FNO) architecture in this work.

\subsection{Problem Formulation}
Given a new input function $f \in \mathcal{X}$, our primary objective is to construct a set-valued predictor $\Gamma_\alpha(f) \subset \mathcal{Y}$ that contains the true, unknown solution $u = \mathcal{G}(f)$ with a user-specified probability $1-\alpha$. We seek to satisfy the classical marginal coverage guarantee of conformal prediction 
$
    \mathbb{P}\left( u \in \Gamma_\alpha(f) \right) \geq 1 - \alpha,
$
where the probability is taken over the draw of the i.i.d. training and calibration data, as well as the unseen test pair $(f, u)$. Achieving this goal in the infinite-dimensional setting of neural operators introduces three core challenges that this paper addresses:
\begin{enumerate}
    \item How can a scalar nonconformity score be defined to meaningfully capture the distance between a predicted function, $\hat{u}$, and the ground truth, $u$? How can this score be designed to distinguish the neural operator's intrinsic prediction error from the error arising from discretizing the underlying continuous function domain?
    \item How does the choice of grid scheme affect the calibrated uncertainty bounds? Furthermore, how can we decompose these bounds to characterize their constituent sources of error?
    \item How can we provide meaningful UQ for autoregressive forecasting tasks where error accumulation violates standard conformal assumptions? Can we adapt metrics from ensemble learning to develop more informative diagnostics that quantify forecast degradation over time?
\end{enumerate}

\section{Theoretical Foundation}\label{sec:try}

To formally extend CP to the function-space we follow a two step method. We first establish finite-sample coverage guarantees in a finite-dimensional, discretized space via the standard split CP framework. Then these guarantees are lifted to the function-space by considering the asymptotic convergence as the discretization is refined.

\subsection{CP in Discretized Function Spaces}

Let \(\{(f_i, u_i)\}_{i=1}^{n+1}\) be an i.i.d. sample from a data-generating distribution on an input-output function space \(\mathcal{X} \times \mathcal{Y}\). We use a discretization operator \( P_d: \mathcal{Y} \to \mathbb{R}^d \) that maps each function \( u_i \in \mathcal{Y} \) to its evaluations on a fixed grid over the domain \(\Omega \subset \mathbb{R}^d\). Given a neural operator \(\mathcal{G}_\theta\) trained on set of points collected with $P_d$, we define a nonconformity score for each sample \((f_i, u_i)\) in a separate calibration set as the distance between the discretized prediction and the discretized ground truth: $s_i = d(P_d(\mathcal{G}_\theta(f_i)), P_d(u_i))$, where $d(\cdot, \cdot)$ can be any vector norm.
Using split CP, we compute the threshold \(\tau_\alpha\). This provides a standard, distribution-free coverage guarantee in the discretized space:  \(\mathbb{P}(s_{n+1} \le \tau_\alpha) \ge 1-\alpha\) where $\mathbb{P}(\cdot)$ captures the randomness in  \(\{(f_i, u_i)\}_{i=1}^{n+1}\). 

\begin{figure}[t]
    \centering
    \includegraphics[width=0.9\linewidth]{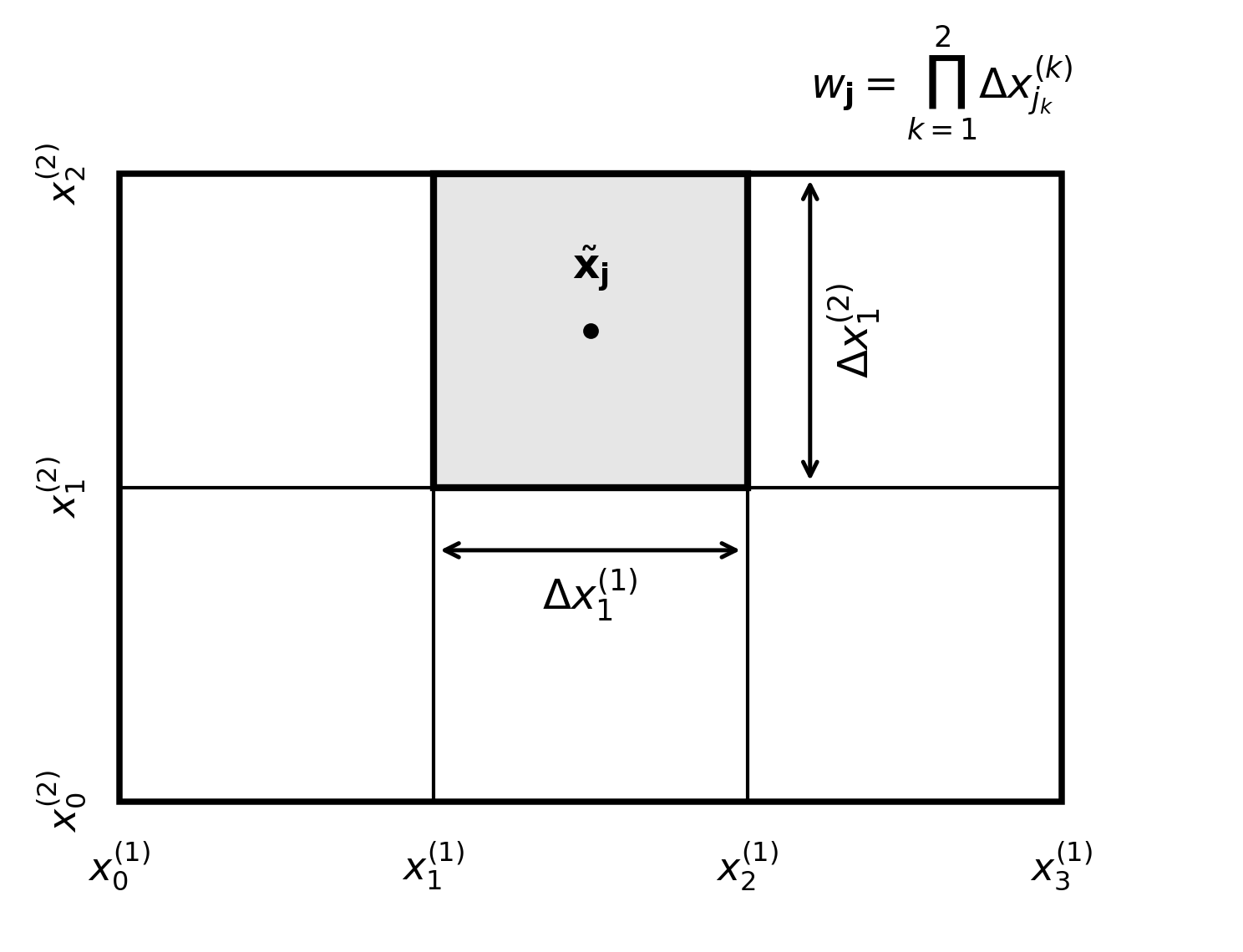}
    \caption{Cell-centered quadrature on a uniform Cartesian grid. 
    The highlighted cell $C_{\mathbf{j}}$ is indexed by $\mathbf{j}=(j_1,j_2)$, 
    with side lengths $\Delta x^{(1)}_{j_1}$ and $\Delta x^{(2)}_{j_2}$, 
    center $\tilde{\mathbf{x}}_{\mathbf{j}}$, and quadrature weight 
    $w_{\mathbf{j}}=\prod_{k=1}^2 \Delta x^{(k)}_{j_k}$.}
    \label{fig:quadrapture}
\end{figure}

\subsection{From Discrete to Continuous Guarantees}

The critical step is to ensure this guarantee translates back to the infinite-dimensional function space \(\mathcal{Y}\). This first requires a formal link between the discrete and continuous geometries. To achieve this, we define the distance metric \(d(\cdot, \cdot)\) using a quadrature-weighted \(L^2\) norm. First, consider a continuous function \(u \in \mathcal{Y}\). We discretize this function using $P_d$ on a structured Cartesian grid over $\Omega \subset \mathbb{R}^d$ to obtain a discrete vector, $h = P_d(u)$, with coordinate vectors $x^{(k)} = \{x^{(k)}_{j_k}\}_{j_k=0}^{N_k}$ along each dimension $k = 1, \dots, d$. For each multi-index $\mathbf{j} = (j_1, \dots, j_d) \in \mathcal{J}$, where $\mathcal{J} = \prod_{k=1}^d \{0, \dots, N_k-1\}$, we define the quadrature weight (i.e., the cell's volume) as
$
    w_{\mathbf{j}} = \prod_{k=1}^d \Delta x^{(k)}_{j_k},
$
where $\Delta x^{(k)}_{j_k} = x^{(k)}_{j_k+1} - x^{(k)}_{j_k}$ denotes the grid spacing in the $k$-th coordinate direction. Therefore the norm is defined as
$
    \|h\|_{w,2}^2 = \sum_{\mathbf{j} \in \mathcal{J}} w_{\mathbf{j}} \, h(\mathbf{x}_{\mathbf{j}})^2,
$
where $\mathbf{x}_{\mathbf{j}} = (x^{(1)}_{j_1}, \dots, x^{(d)}_{j_d})$ denotes the grid point corresponding to multi-index $\mathbf{j}$ over the $d$-dimensional domain. This process is detailed in Figure~\ref{fig:quadrapture}
 By construction, this is a Riemann sum for the integral of \(u^2\) over the domain \(\Omega\). Consequently, it provably converges to the continuous \(L^2(\Omega)\) norm as the discretization is refined (see Theorem~4 in the appendix for a formal proof). Next, to connect the discrete and continuous spaces, the projection \(P_d\) must preserve the geometry of the function space. That is, the distance between two functions in the continuous space must be comparable to the distance between their discretized representations. Next, we assume that the target function \(u \in \mathcal{Y}\) belongs to a Sobolev space \(H^s(\Omega)\) with \(s > d/2\) \cite{le2024mathematical, katende2025stability}. Since solutions generated by neural operators to partial differential equations (PDEs) are typically smooth, this assumption is justified \cite{furuya2024quantitative}. This motivates the following:
\begin{assumption}[Bilipschitz Discretization]
\label{assum:bilipschitz}
The discretization map \(P_d\) is bilipschitz, meaning there exist constants \(0 < c_1 \le 1 \le c_2\) such that for all \(u, v \in \mathcal{Y}\):
\begin{equation}
    c_{1}\|u-v\|_{\mathcal{Y}} \le \|P_d(u)-P_d(v)\|_{w,2,d} \le c_{2}\|u-v\|_{\mathcal{Y}}
\end{equation}
\end{assumption}
\noindent Here, the constants $c_1$ and $c_2$ quantify this approximation error, which for smooth functions on a grid with mesh size $\eta$ is known to scale as $1 \pm \mathcal{O}(\eta^k)$ for some $k>0$. And again, since the PDE solutions we are interested in are smooth, we expect $c_1 \to 1$ for a sufficiently fine grid, making our discrete norm a near-isometric proxy for the continuous one \cite{hicken2013summation}. 
\begin{remark}\label{rmrk:c1remark}
In practice, such inequalities hold when $P_d$ is implemented as (i) a sufficiently fine grid-based sampling of $u\in\mathcal{Y}$, or (ii) a truncated spectral expansion that retains enough terms to capture the relevant energy in $\|\cdot\|_{\mathcal{Y}}$. Moreover, when $d \gg N$, the constants $c_1$ and $c_2$ are approximatly 1, implying that $P_d$ becomes a near-isometric embedding of the subspace of interest.
\end{remark}
\noindent Under this assumption, we derive our main theoretical result.
\begin{theorem}[Functional Conformal Coverage]
\label{thm:functional_coverage}
Let \(\tau_\alpha\) be the threshold calibrated on the discrete scores \(s_i = \|P_d(\hat{u}_i) - P_d(u_i)\|_{w,2,d}\) where $\hat{u}_i$ is the predicted function. Under Assumption~\ref{assum:bilipschitz}, the functional prediction set:
\begin{equation}
    \Gamma_{\alpha}^{\text{func}}(f) = \{v \in \mathcal{Y} : \|\mathcal{G}_{\theta}(f) - v\|_{\mathcal{Y}} \le \tau_{\alpha}/c_1\},
\end{equation}
satisfies the coverage guarantee:
\begin{equation}
    \mathbb{P}\left( u_{n+1} \in \Gamma_{\alpha}^{\text{func}}(f_{n+1}) \right) \ge 1 - \alpha,
\end{equation}
where \((f_{n+1}, u_{n+1})\) is an i.i.d. test point drawn from the data-generating distribution.

\end{theorem}
\begin{proof}[Proof Sketch]
The proof proceeds by connecting the established discrete-space guarantee to the continuous function space via Assumption~\ref{assum:bilipschitz}.
From split CP, we have the guarantee \(\mathbb{P}(s_{n+1} \le \tau_\alpha) \ge 1-\alpha\) in the discrete space, where \(s_{n+1} = \|P_d(\hat{u}_{n+1}) - P_d(u_{n+1})\|_{w,2,d}\).
The left-hand side of the inequality in Assumption~\ref{assum:bilipschitz} states that \(c_1 \|\hat{u}_{n+1} - u_{n+1}\|_\mathcal{Y} \le s_{n+1}\).
Combining these, the event \(s_{n+1} \le \tau_\alpha\) implies the event \(c_1 \|\hat{u}_{n+1} - u_{n+1}\|_\mathcal{Y} \le \tau_\alpha\), which is equivalent to \(\|\hat{u}_{n+1} - u_{n+1}\|_\mathcal{Y} \le \tau_\alpha/c_1\).
Since the first event implies the second, the probability of the second event must be at least as large as the first: \(\mathbb{P}(\|\hat{u}_{n+1} - u_{n+1}\|_\mathcal{Y} \le \tau_\alpha/c_1) \ge \mathbb{P}(s_{n+1} \le \tau_\alpha) \ge 1-\alpha\). This is precisely the coverage guarantee for the functional prediction set \(\Gamma_{\alpha}^{\text{func}}(f)\).
The full proof is provided in the appendix.
\end{proof}

\subsection{A Heuristic Model for the Conformal Radius}\label{sec:super}

Similar to the classic bias-variance decomposition of prediction error \cite{brofos2019bias}, we view the conformal radius \(\tau_\alpha(d)\)--where $d$ denotes the number of evaluation points used by the discretization operator $P_d$--as arising from three dominant sources: discretization of the underlying function, finite-sample calibration, and model misspecification. To capture these effects, we propose a first-order~\cite{hastie2009elements} heuristic model that decomposes \(\tau_\alpha(d)\) as:
\begin{equation}
\tau_\alpha(d) \approx \underbrace{\varepsilon_{\text{disc}}(d)}_{\text{discretization}}
+ \underbrace{\varepsilon_{\text{cal}}}_{\text{calibration}}
+ \underbrace{\varepsilon_{\text{misspec}}(d)}_{\text{misspecification}}.
\end{equation}
Here, \( \varepsilon_{\text{disc}}(d) = \|u - P_d u\|_{w, 2} \) is the discretization error, decaying as \( \mathcal{O}(d^{-p}) \) \cite{lanthaler2024discretizationerrorfourierneural}, \( \varepsilon_{\text{cal}} = \mathcal{O}(1/\sqrt{n}) \) is the finite-sample calibration error \cite{ghosh2023probabilisticallyrobustconformalprediction}, and \( \varepsilon_{\text{misspec}}(d) \) is the model's generalization error at resolution \( d \). Due to the dependence of $\varepsilon_{\text{misspec}}$ on the predictor, developing a theoretical bound on its error is challenging and inefficient. Instead, we analyze the distribution of $\tau_\alpha$ values across resolutions, per predictors. Empirically, we find the distributions of $\tau_\alpha$ are approximately log-linear evaluated at resolutions beyond the training resolution when using a FNO, making it particularly useful for super-resolution tasks without requiring retraining or recalibration (see Figures 5 and 6 in the Appendix). To estimate the super-resolution conformal radius \(\tau_\alpha\), we fit a regression of the form
\begin{equation}
\tau(R) = \exp\left( s \cdot R + b \right),
\end{equation}
where \(\log \tau(R_i) = s \cdot R_i + b\). Although no formal coverage guarantee exists for the extrapolated value \(\tau_\alpha\), we find that it yields substantial improvements in coverage accuracy, detailed in Table~\ref{tab:darcy_poisson_coverage}.

\subsection{Time-Series Forecasting}
Until now, we considered spatial mappings of the form \( f \mapsto u \), where split conformal prediction guarantees hold under exchangeability. We now extend this perspective to temporal mappings, where a neural operator predicts the evolution over time $t$. 
Formally, let \( u_t \) denote the true system state at time \( t \) and \( \hat{f}_t \) the operator’s forecasted input at time \( t \). Because each forecasted state \(\hat{f}_t\) depends on previous predictions, the sequence \(\{(\hat{f}_t, u_t)\}_{t=1}^T\) is no longer exchangeable. This violates the core assumption of split conformal prediction, causing a coverage gap as prediction errors compound over time. Rather than “correcting” the conformal radius to maintain a guarantee under drift \cite{cleaveland2024conformal}, we reinterpret the coverage gap itself as a diagnostic signal. 
\begin{theorem}[Drift-Aware Functional Coverage]
\label{thm:drift_functional_coverage}
Let $\tau_\alpha$ be the conformal threshold computed on discrete nonconformity scores 
$s_i = \| P_d(\hat{u}_i) - P_d(u_i) \|_{w,2,d}$ under a calibration distribution $\mathbb{P}_{\mathrm{cal}}$,
where $P_d: \mathcal{Y} \to \mathbb{R}^d$ is bilipschitz as in Assumption~\ref{assum:bilipschitz}.
At forecast steps $t = 1, \dots, T$, let the data distribution drift to $\mathbb{P}_t$.
Define the functional prediction set:
$
\Gamma_{\alpha}^{\text{func}}(f_t) = \big\{ v \in \mathcal{Y} : \| \hat{u}_t - v \|_{\mathcal{Y}} \le \tau_\alpha / c_1 \big\},
$
where $c_1$ is the constant from Assumption~\ref{assum:bilipschitz}.
Then, for each $t$,
\begin{equation}
\mathbb{P}_t \left( u_t \in \Gamma_{\alpha}^{\text{func}}(f_t) \right) \ge 1 - \alpha - d_{\mathrm{TV}}(\mathbb{P}_{\mathrm{cal}}, \mathbb{P}_t),
\end{equation}
where $d_{\mathrm{TV}}$ is the total variation distance between the calibration and forecast-time distributions. 
\end{theorem}
\noindent By applying a fixed, conformally-calibrated threshold at each forecast step, a violation of the conformal bound serves as an interpretable signal that the accumulated model drift, quantified by \( d_{\mathrm{TV}} \), has become significant enough to degrade coverage below the desired level.

\section{Methodology and Implementation}
This section details our methodology for functional UQ, simultaneous pointwise UQ, and autoregression diagnostic metrics. We construct calibrated prediction sets for three settings: sampling-based epistemic models, quantile-based predictions, and autoregressive stochastic forecasts.

\subsection{Conformal Calibration}
Following supervised training, we perform conformal calibration on a held‑out set of input–output function pairs \(\{(f_i, u_i)\}_{i=1}^n\). Each ground‑truth function \(u_i^{(d)} = P_d(u_i)\) and prediction \(\hat{u}_i^{(d)} = P_d(\mathcal{G}_\theta(f_i))\) are discretized via the operator \(P_d\). The nonconformity score for the \(i\)-th calibration sample is the relative weighted \(L^2\) error 
$
s_i = \frac{\| \hat{u}_i - u_i \|_{w,2}}{\| \hat{u}_i \|_{w,2}},
$
where \(\| \cdot \|_{w,2}\) is the quadrature‑weighted norm. We use the relative form so the score is scale‑invariant across functions of different magnitudes and therefore more meaningful. Given we are working with fine resolutions we assume \(c_1=1\), following Remark~\ref{rmrk:c1remark}. Using this score, we compute the threshold \(\tau_\alpha\), and the prediction set for any new input function \(f\) is:
\begin{equation}
\Gamma_\alpha(f) = \left\{ u : \frac{\| \hat{u} - u \|_{w,2}}{\| \hat{u} \|_{w,2}} \le \tau_\alpha \right\}.
\end{equation}

\subsubsection{Monte Carlo Bounding}
Generative methods characterize uncertainty by producing multiple realizations of the solution. Popular examples include MC-dropout, variational autoencoders, and diffusion models, which generate an ensemble of \(n\) candidate realizations \(\hat{u}^{(j)}(x) \sim S_\theta(x)\). The spread of these samples is then used to quantify variability. To construct pointwise prediction intervals, we compute lower and upper bounds at each spatial location \(j\) by taking the minimum and maximum over the ensemble, forming a conservative bounding envelope
$\hat{u}^{\min}(x) = \min_{1 \le j \le n} \hat{u}^{(j)}(x)$, and $\hat{u}^{\max}(x) = \max_{1 \le j \le n} \hat{u}^{(j)}(x)$.
Under this formulation, we obtain conservative estimators of the functional upper and lower bounds directly at inference time. An alternative approach would involve using the sample standard deviation to form prediction bands; however, we find that this often leads to undercoverage. To incorporate the conformal threshold, we modify the previous procedure by conditioning the interval construction on the calibrated threshold \(\tau_\alpha\), such that $\hat{u}^{\min}(x) = \min_{1 \le j \le n} \hat{u}^{(j)}(x) | \hat{u}(x) \leq \tau_\alpha$, and $\hat{u}^{\max}(x) = \max_{1 \le j \le n} \hat{u}^{(j)}(x) | \hat{u}(x) \leq \tau_\alpha$.

\begin{table}[!b]
\caption{Coverage Summary for Darcy 1D. Both scalars are set for a significance level $\alpha = 0.1$. Our method calibrates the overly conservative baselines to achieve our desired significance level.}
\centering
\small
\renewcommand{\arraystretch}{1.2}
\setlength{\tabcolsep}{5pt}
\begin{tabular}{llccc}
\toprule
 & \textbf{Technique} & \textbf{Scalar} & \textbf{Func.} & \textbf{Point.} \\
\midrule
\multirow{2}{*}{\textbf{Uncalibrated}} 
& MC Dropout & -- & -- & 1.000 \\
& Variational & -- & -- & 1.000 \\
\midrule
\multirow{2}{*}{\textbf{Calibrated}} 
& MC Dropout & 0.0044 & 0.9003 & \textbf{0.9200} \\
& Variational & 0.0035 & \textbf{0.9000} & 0.8807 \\
\bottomrule
\end{tabular}
\label{tab:darcy_coverage}
\end{table}

\begin{figure*}[ht]
    \centering
    \small
    \setlength{\tabcolsep}{4pt}
    \renewcommand{\arraystretch}{1.0}

    \begin{tabular}{@{}c@{}c@{}c@{}c@{}}
        % First row (MC Dropout)
        \raisebox{19mm}{\rotatebox[origin=c]{90}{\textbf{MC Dropout}}} &
        \begin{minipage}[b]{0.31\textwidth}
            \centering
            \caption*{\hspace{2mm}(a) Prediction vs Ground Truth}
            \includegraphics[width=0.9\textwidth]{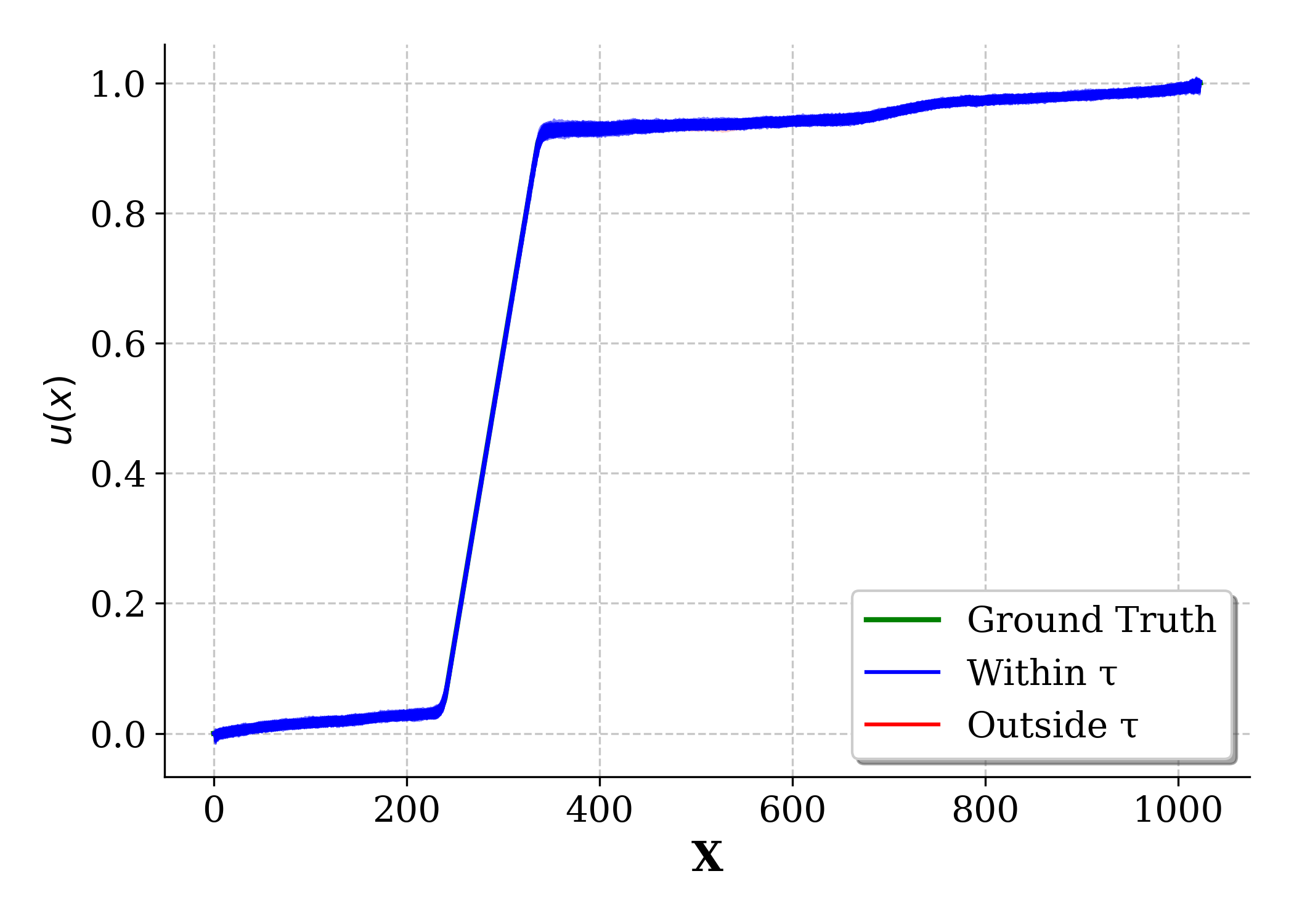}
        \end{minipage} &
        \begin{minipage}[b]{0.31\textwidth}
            \centering
            \caption*{\hspace{2mm}(b) Prediction Interval}
            \includegraphics[width=0.9\textwidth]{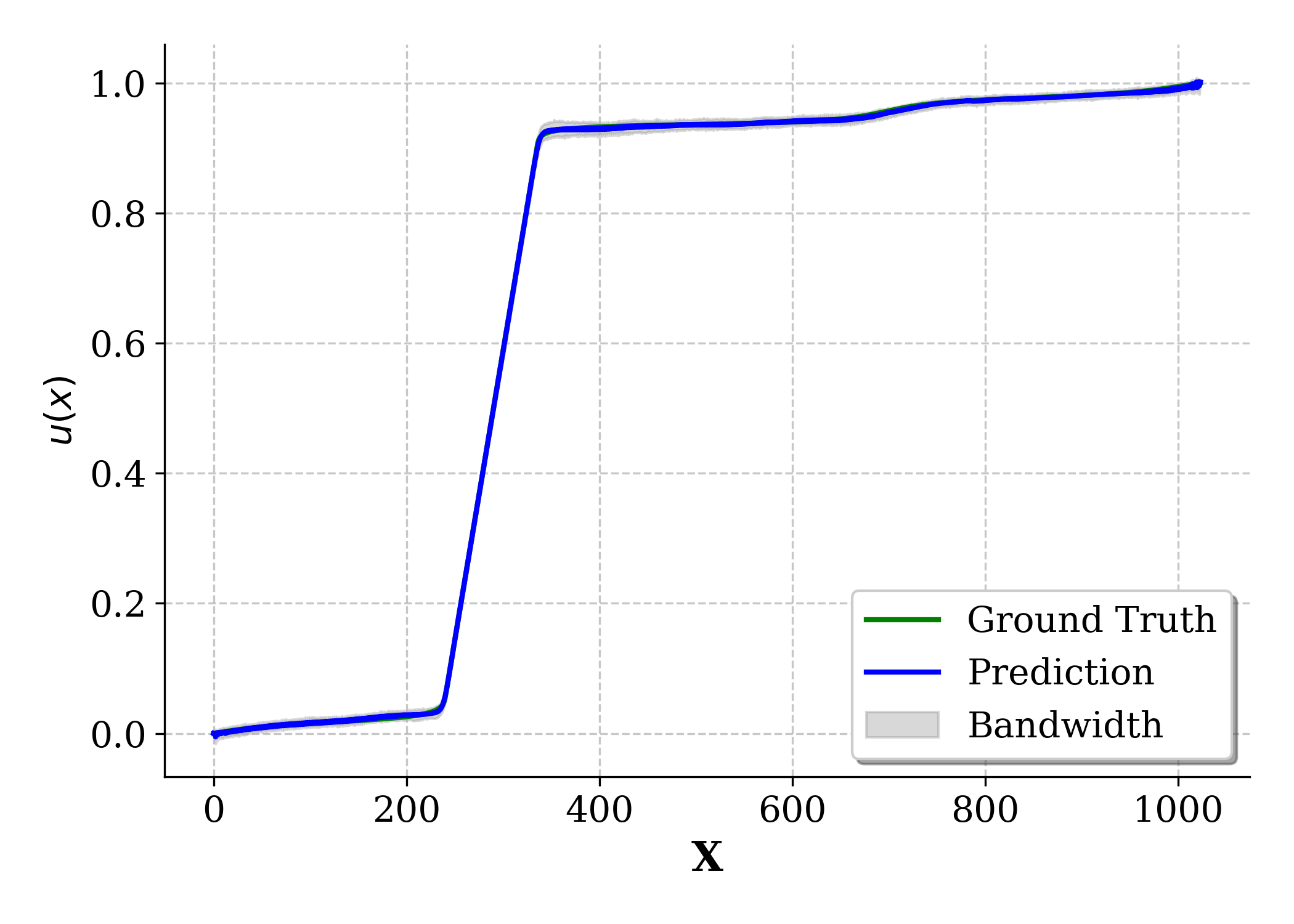}
        \end{minipage} &
        \begin{minipage}[b]{0.31\textwidth}
            \centering
            \caption*{\hspace{5mm}(c) Std. vs Prediction Error}
            \includegraphics[width=0.9\textwidth]{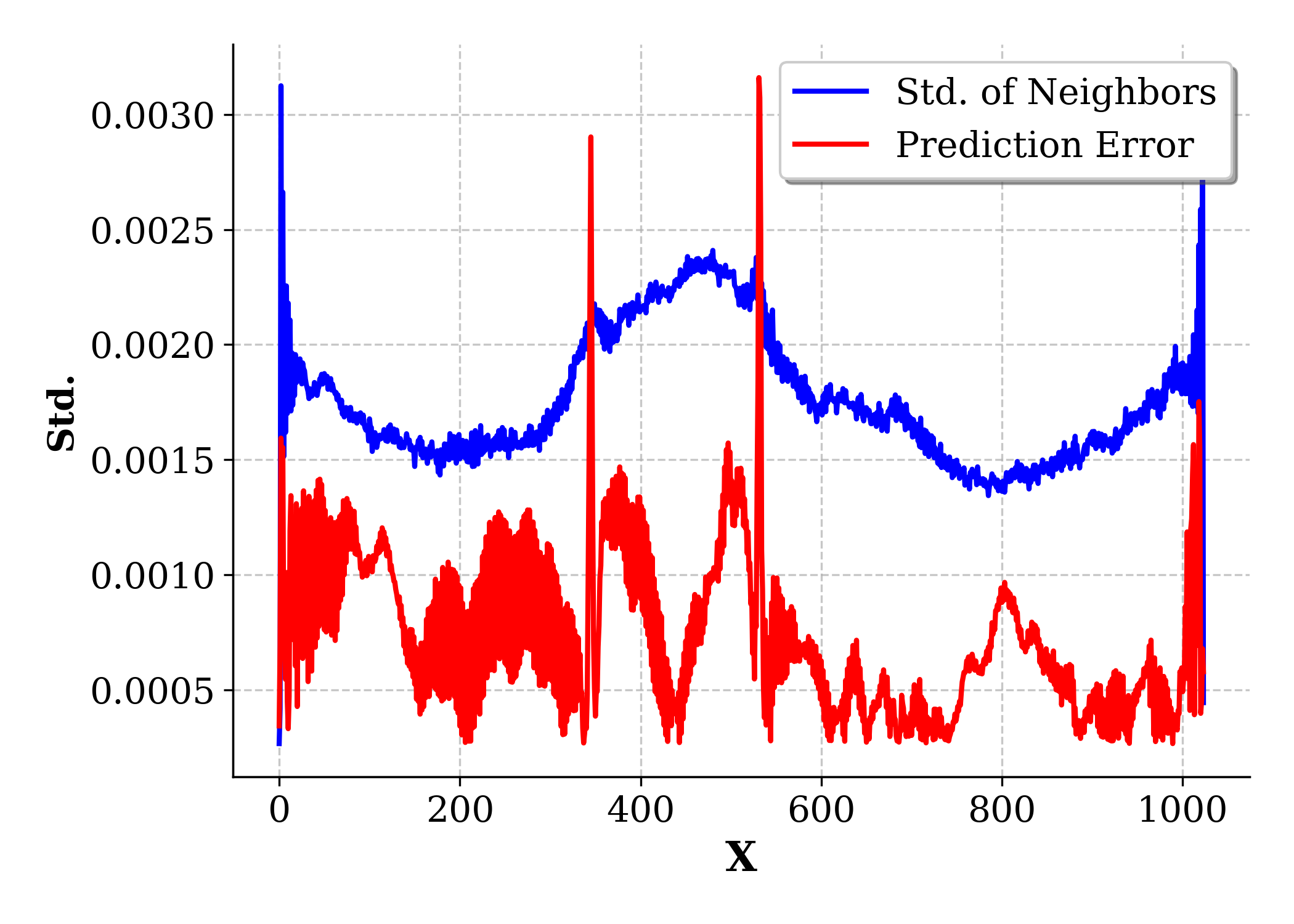}
        \end{minipage} \\[2mm]

        % Second row (VAE)
        \raisebox{19mm}{\rotatebox[origin=c]{90}{\textbf{Variational}}} &
        \begin{minipage}[b]{0.31\textwidth}
            \centering
            \includegraphics[width=0.9\textwidth]{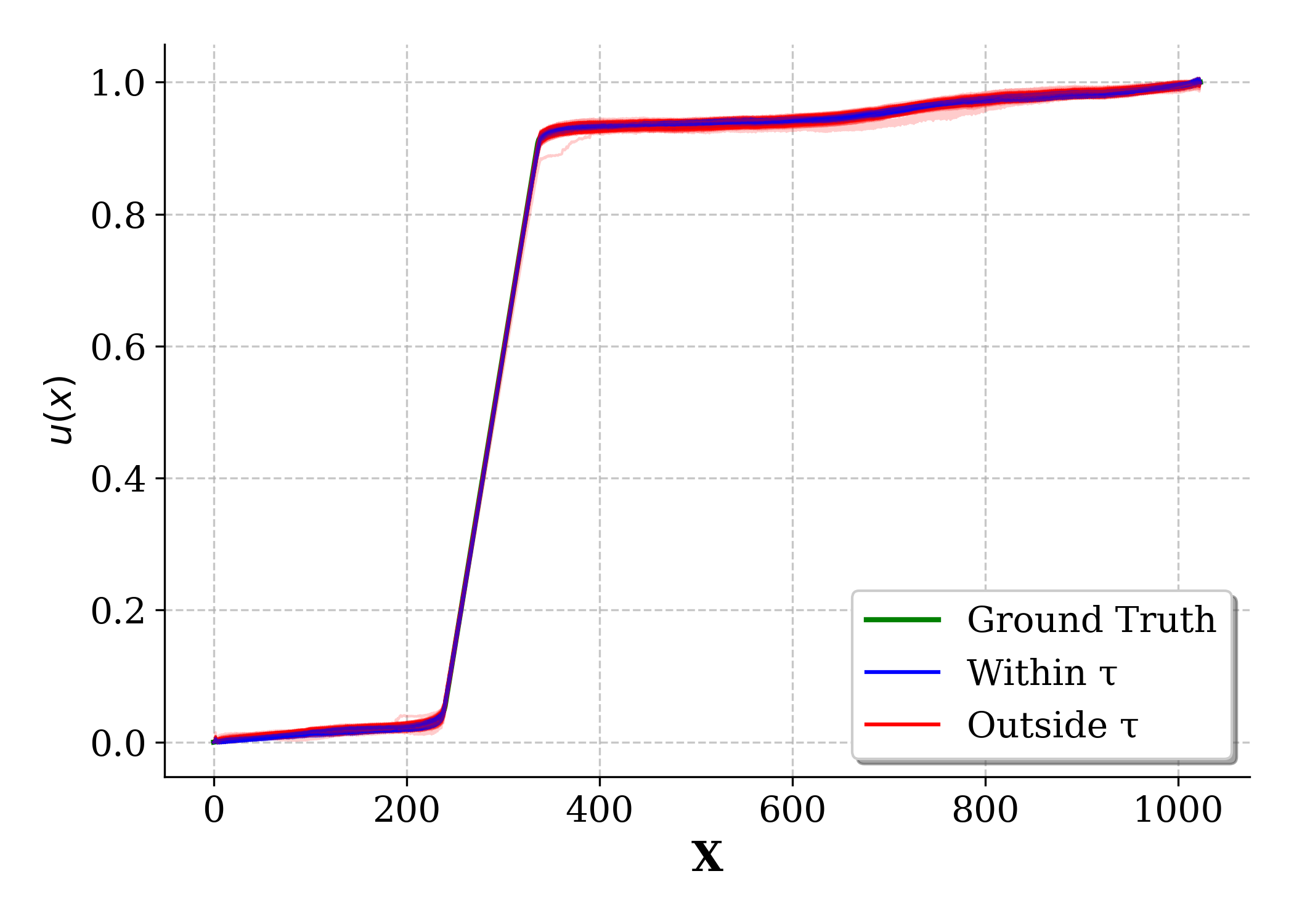}
        \end{minipage} &
        \begin{minipage}[b]{0.31\textwidth}
            \centering
            \includegraphics[width=0.9\textwidth]{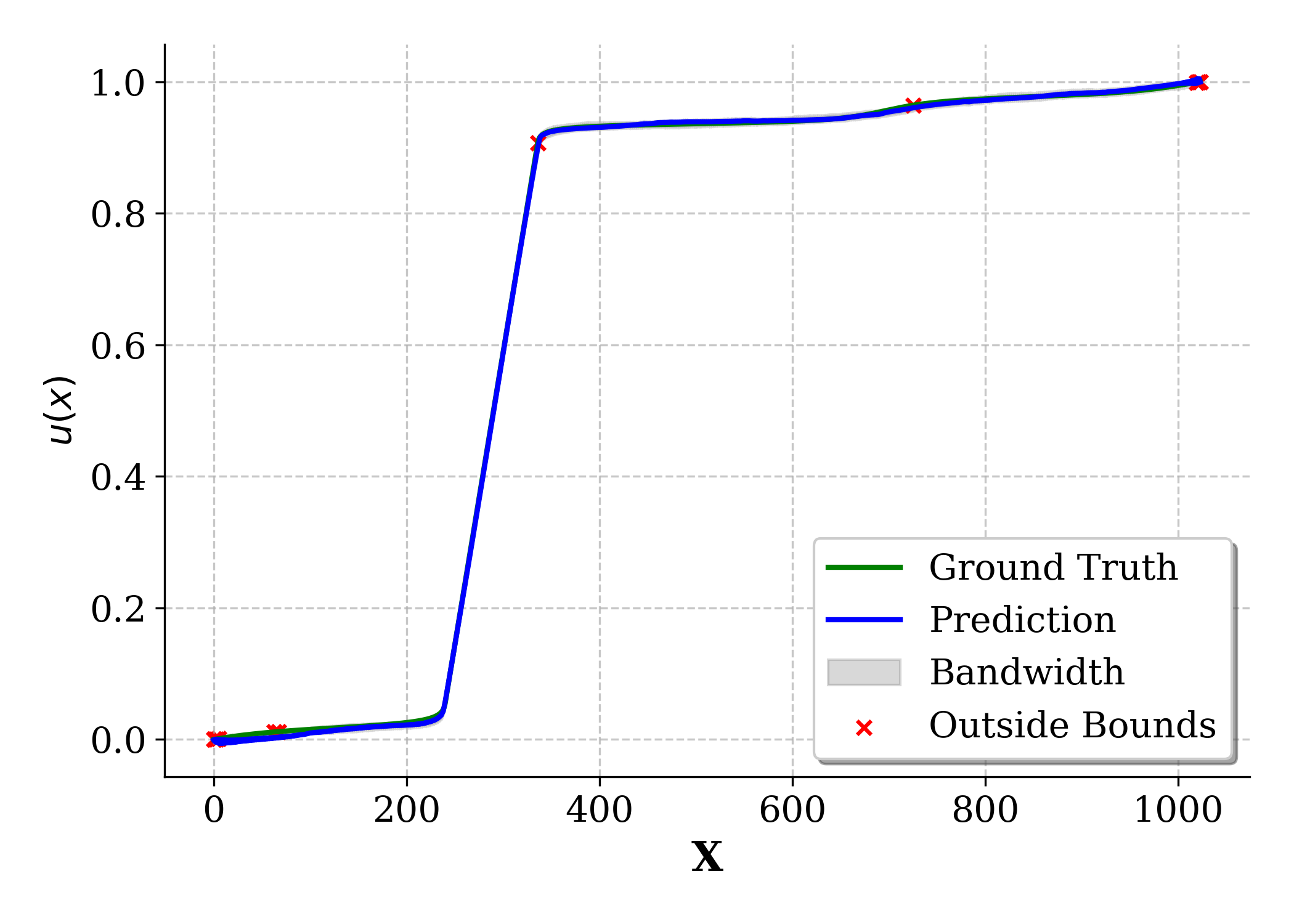}
        \end{minipage} &
        \begin{minipage}[b]{0.31\textwidth}
            \centering
            \includegraphics[width=0.9\textwidth]{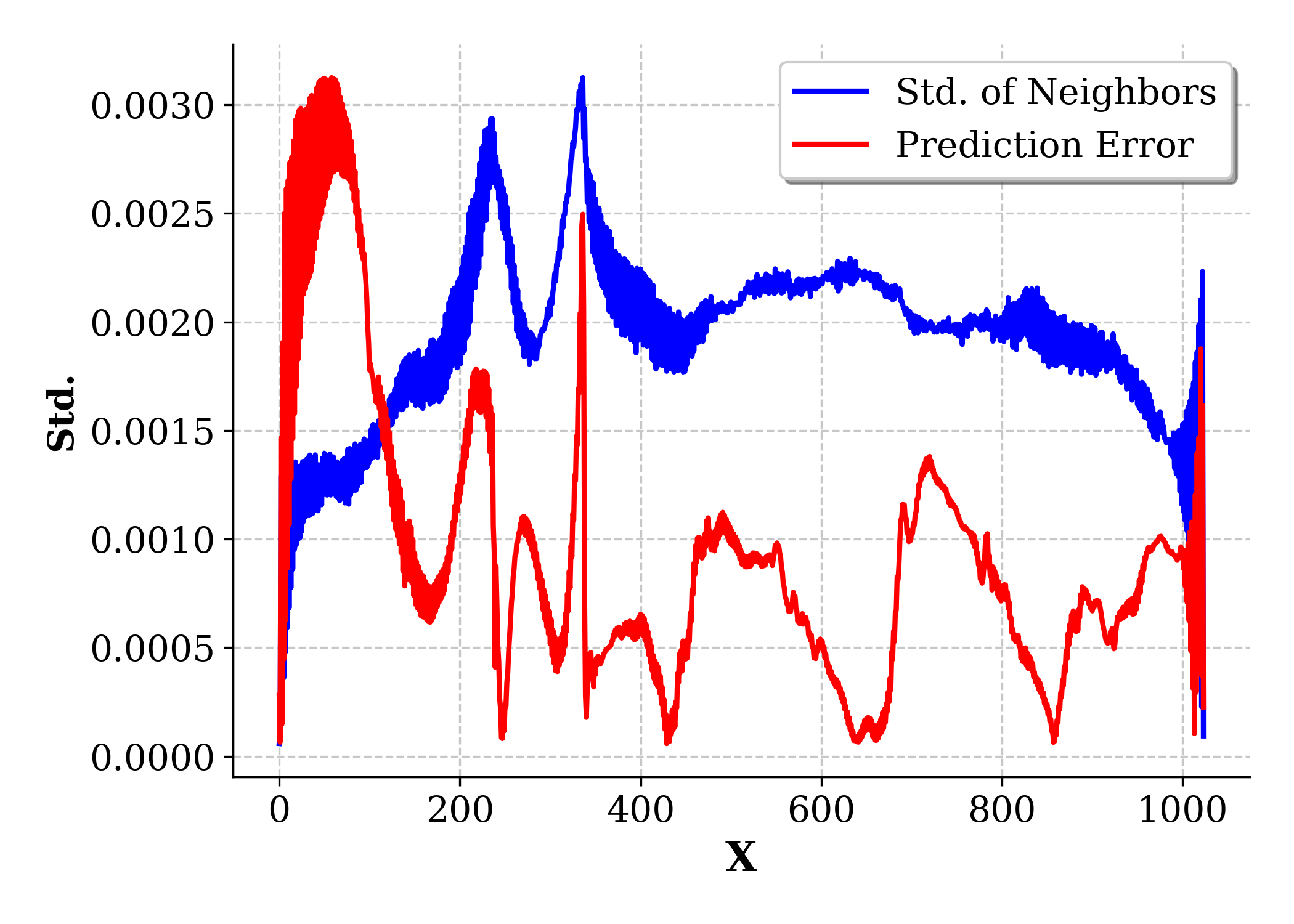}
        \end{minipage}
    \end{tabular}

    \caption{Darcy 1D: MC Bounds, with $\tau_\alpha$ calibration at level $\alpha = 0.1$. Shown is a sample drawn and evaluated on a FNO with MC-Dropout and a variational-FNO. \textbf{(a)} Samples drawn for one specific instance, with samples within $\tau_\alpha$ shown in blue and those outside $\tau_\alpha$ in red. \textbf{(b)} Bounds inferred from the minimum and maximum of the samples within $\tau_\alpha$. \textbf{(c)} Bandwidth versus true error.}
    \label{fig:darcy_ex}
\end{figure*}

\subsubsection{Quantile Bounds Adjustment}
Alternatively, models may learn a direct heuristic of the error field, thereby producing explicit estimates of the prediction interval. In this work, we consider the triplet formulation \((\hat{u}^{\text{lo}}, \hat{u}^{\text{mid}}, \hat{u}^{\text{hi}})\), where the upper and lower bounds are learned via quantile regression using the pinball loss \cite{huang2013support}. To align the resulting interval \([\hat{u}^{\text{lo}}, \hat{u}^{\text{hi}}]\) with the target coverage level, we perform a post hoc adjustment. For a given bound \(\hat{u}\) (either lower or upper), we define its offset from the central prediction as \(\delta = \hat{u} - \hat{u}^{\text{mid}}\), and compute the initial error:
\begin{equation}
r_{\text{initial}} = \frac{\|\hat{u} - \hat{u}^{\text{mid}}\|}{\|\hat{u}^{\text{mid}}\|} = \frac{\|\delta\|}{\|\hat{u}^{\text{mid}}\|}.
\end{equation}
Next, we derive a scaling factor $s$ as the ratio of the target relative error $\tau_\alpha$ to the initial error:
$\mathbf{s} = \tau_{\alpha} / r_{\text{initial}}$.
Finally, adjusted bound $\hat{u}^{\text{adj}}$ is then found by scaling the original offset:
\begin{equation}
\hat{u}^{\text{adj}} = \hat{u}^{\text{mid}} + s \cdot \delta = \hat{u}^{\text{mid}} + \left( \frac{\tau_\alpha}{r_{\text{initial}}} \right) \cdot (\hat{u} - \hat{u}^{\text{mid}}).
\end{equation}
This procedure, applied to both $\hat{u}^{\text{lo}}$ and $\hat{u}^{\text{hi}}$, yields a recalibrated interval $[\hat{u}^{\text{min}}, \hat{u}^{\text{max}}]$ whose bounds are exactly $\tau_\alpha$ distant from the center in relative terms.

\begin{figure*}[ht]
    \centering
    \small
    \begin{minipage}[b]{0.32\textwidth}
        \centering
        \includegraphics[width=0.9\textwidth]{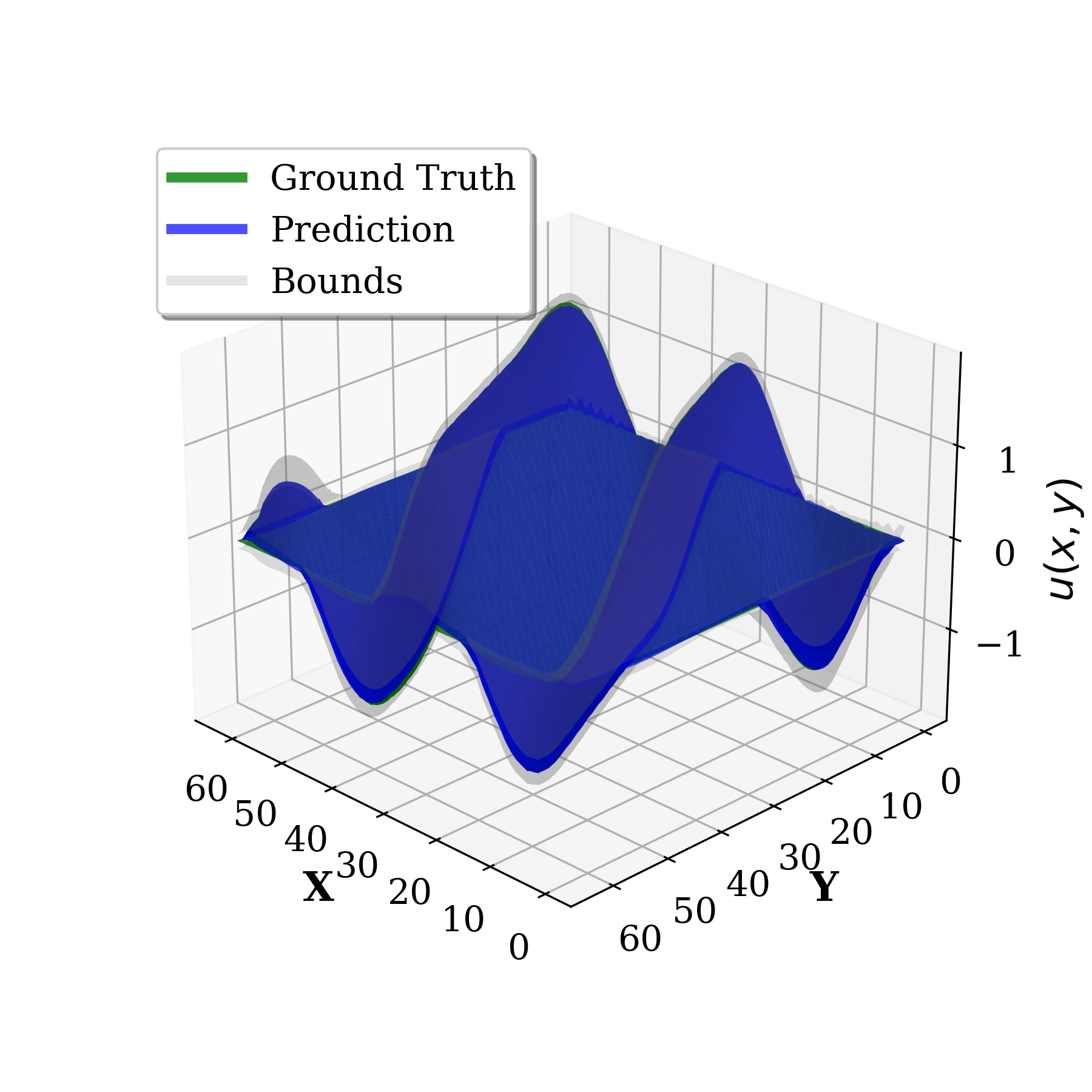}
        \caption*{\hspace{2mm}(a) Prediction vs Ground Truth}
    \end{minipage}
    \hfill
    \begin{minipage}[b]{0.32\textwidth}
        \centering
        \includegraphics[width=0.9\textwidth]{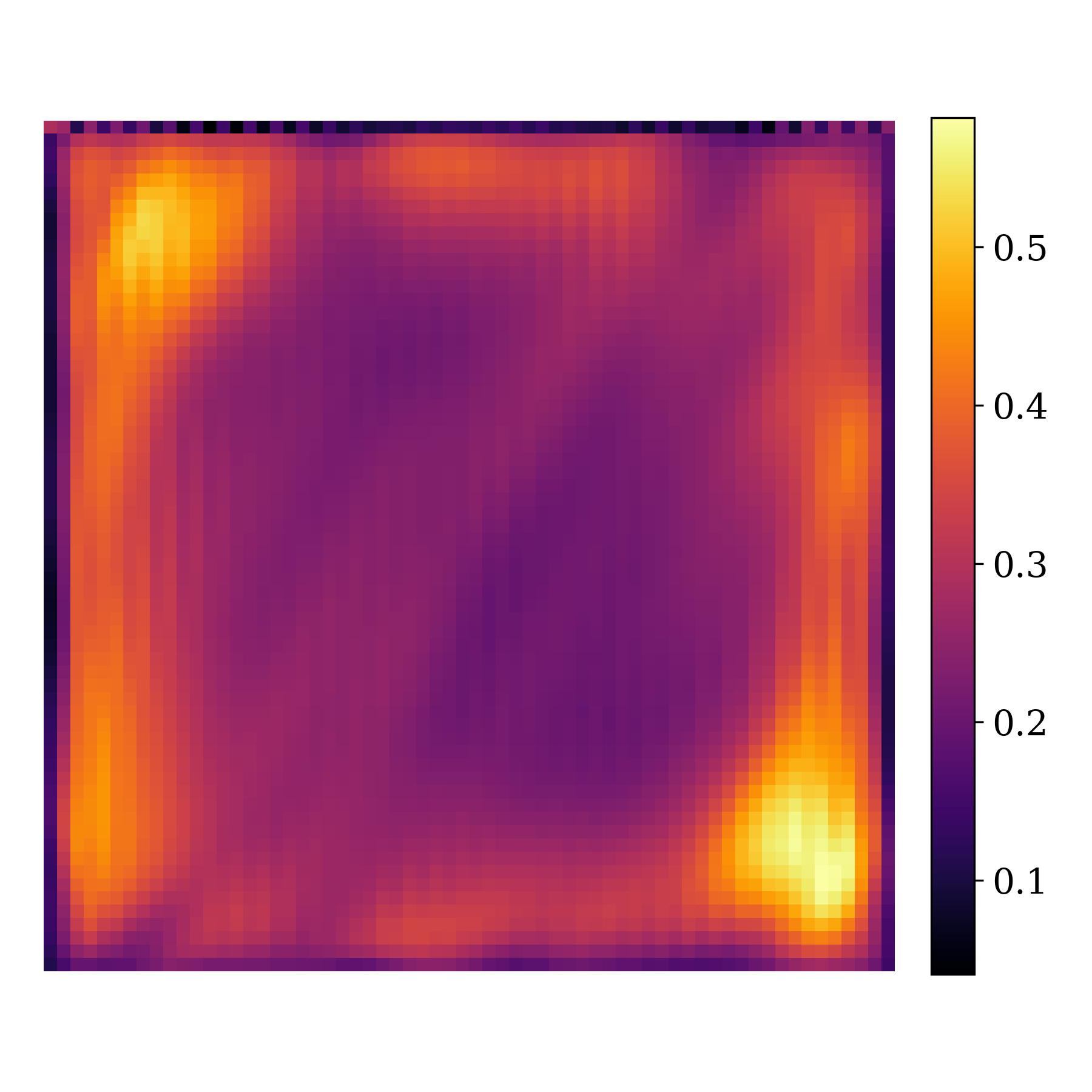}
        \caption*{\hspace{-6mm}(b) Std. of Neighbors}
    \end{minipage}
    \hfill
    \begin{minipage}[b]{0.32\textwidth}
        \centering
        \includegraphics[width=0.9\textwidth]{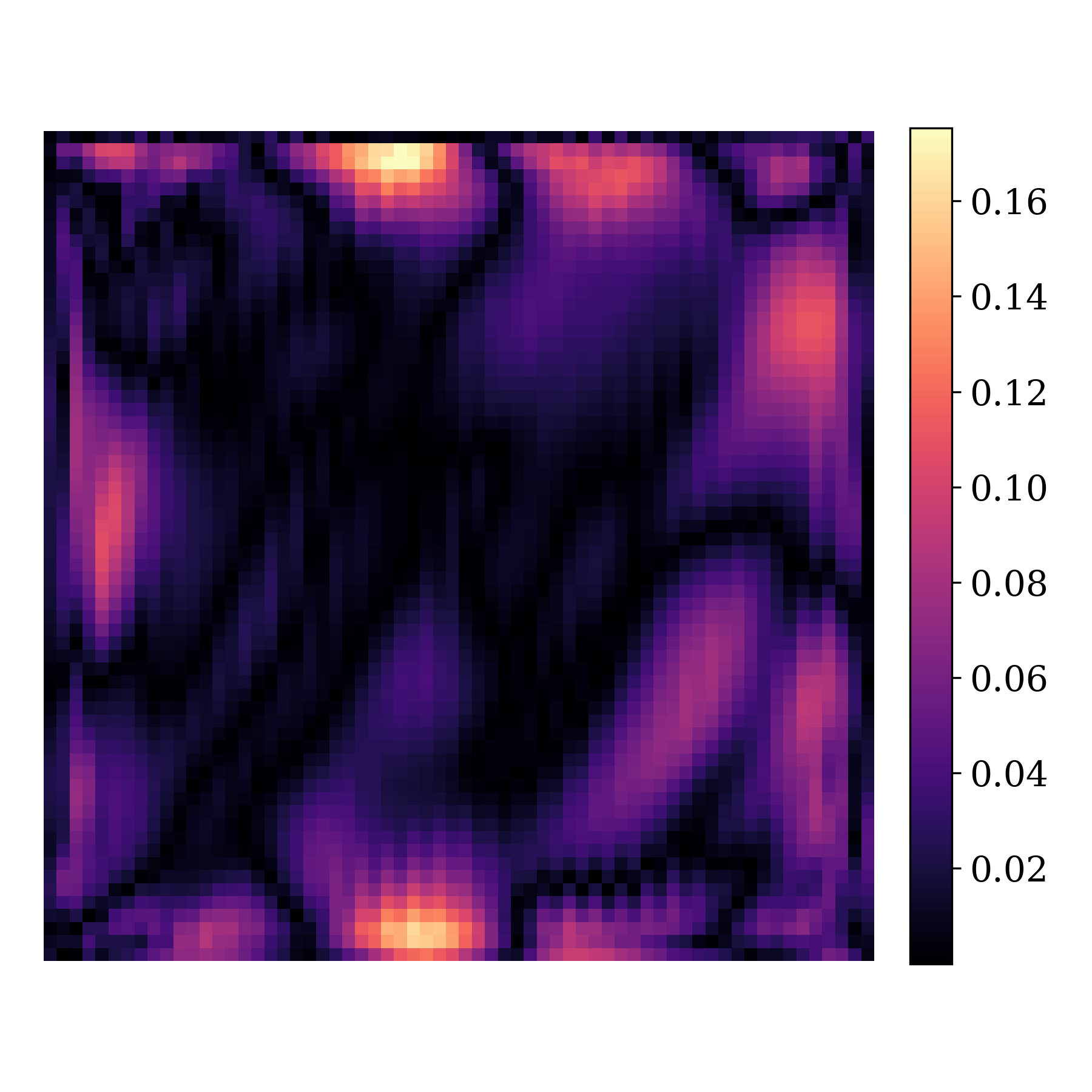}
        \caption*{\hspace{-6mm}(c) Prediction Error}
    \end{minipage}
    \caption{Poisson 2D: Adjusted Quantiles, with $\tau_\alpha$ calibration at level $\alpha = 0.1$. \textbf{(a)} Mean prediction against the ground truth, with upper and lower bounds scaled to the $\tau_\alpha$ distance. \textbf{(b)} Bandwidth of the scaled bounds. \textbf{(c)} Prediction error between the mean prediction and the ground truth.}
    \label{fig:poisson_ex}
\end{figure*}

\subsubsection{Stochastic Forecast Bounding}
For time-series models that generate stochastic trajectories, we first obtain the conformal radius $\tau_\alpha$, calibrated on the one initial one step distribution. We then apply this radius across all forecast steps to monitor for performance degradation. At each time step $t$, the model produces an ensemble of forecasted solutions, $\{ \hat{u}^{(j)}_t \}_{j=1}^n$. The forecast's quality is then evaluated using a the following metrics. We first define the calibration error as the normalized distance between the ensemble mean, $\bar{u}_t$, and the ground-truth solution, $u_t$. This term is directly comparable to $\tau_\alpha$ and provides a binary assessment of whether the forecast remains within its initial bounds. Next, we compute the ensemble spread to quantify forecast sharpness by measuring the average normalized distance of individual ensemble members from their mean:
\begin{equation}
\text{ES}_t = \frac{1}{n} \sum_{j=1}^{n} \frac{\| \hat{u}^{(j)}_t - \bar{u}_t \|}{\|\bar{u}_t\|}, \quad \text{for} \quad \bar{u}_t := \frac{1}{n} \sum_{j=1}^n \hat{u}^{(j)}_t.
\end{equation}
These two components are summed to form the conformal ensemble score (CES), a metric analogous to the continuous ranked probability score (CRPS) \cite{Pic_2023}. Finally, to measure the ensemble's internal consensus, we define the conformal prediction set $\Gamma_\alpha^{(t)}$ as:
\begin{equation}
\Gamma_\alpha^{(t)} = \left\{ \hat{u}^{(j)}_t : \frac{\| \hat{u}^{(j)}_t - \bar{u}_t \|}{\|\bar{u}_t\|} \leq \tau_\alpha \right\}.
\end{equation}
The proportion of members within $\Gamma_\alpha^{(t)}$ serves as our internal agreement (IA) metric. A drop in IA over time reflects a loss of forecast confidence and growing internal divergence. While no formal coverage guarantee holds at individual time steps $t> 1$ due to non-exchangeability in the time-series sequence, this approach provides an estimate of the temporal degradation.

\begin{table}[!b]
\centering
\small
\renewcommand{\arraystretch}{1.0}
\setlength{\tabcolsep}{5pt}
\caption{Coverage Summary for Poisson 2D. Both scalars are set for a significance level $\alpha = 0.1$. The functional and pointwise coverage reaches slightly above our target coverage of 90\%.}
\begin{tabular}{llccc}
\toprule
& \textbf{Method} & \textbf{Scalar} & \textbf{Func.} & \textbf{Point.} \\
\midrule
\multirow{1}{*}{\textbf{Uncalibrated}} 
& Pinball Only & --      & --       & 0.6911 \\
\midrule
\multirow{2}{*}{\textbf{Calibrated}} 
& Risk  Controlling   & 0.1480  & --  & \textbf{0.9101} \\
& Our Bounds  & 0.1486  & \textbf{0.9010} & 0.9108 \\
\bottomrule
\end{tabular}
\label{tab:poisson_coverage}
\end{table}

\subsubsection{Evaluation}
We assess calibration at three levels. First, pointwise coverage measures the proportion of spatial points within the conformal bounds, 
\( C_e = \frac{1}{N} \sum_{i=1}^{N} \mathbf{I}[ \hat{u}^{\min}_i(x) \le u_i(x) \le \hat{u}^{\max}_i(x) ] \). 
Second, functional coverage evaluates whether entire functions satisfy the calibrated radius, 
\( C_f = \frac{1}{K} \sum_{j=1}^{K} \mathbf{I}\left[ \frac{\|\hat{u}_j - u_j\|_{w,2}}{\|\hat{u}_j\|_{w,2}} \le \tau_\alpha \right] \). 
Finally, for long-horizon forecasts, we track diagnostic metrics (CES and IA) and qualitatively interpret their signals.

\section{Numerical Results \& Discussion}
To validate our framework, we evaluate the resulting prediction sets on progressively more challenging PDE benchmarks: Monte Carlo bounding for Darcy flow, post hoc quantile adjustment for the Poisson equation, and bounded ensemble forecasting for autoregressive Navier–Stokes dynamics. In addition, we conduct two ablation studies to assess the framework’s robustness to grid discretization and resolution misspecification.

\subsection{Darcy 1D: Monte Carlo Bounds}
We first evaluate ability of MC methods to produce calibrated bounds for generative UQ in 1D Darcy flow using variational and MC-dropout neural operators. Shown in Table~\ref{tab:darcy_coverage}, the uncalibrated approach yields overly consecrative coverage, whereas our calibration procedure successfully produces uncertainty sets that meet the target coverage. Additionally, we highlight a trade-off between methods: MC-Dropout achieves robust coverage with a larger conformal radius, while the variational method yields a tighter bound at the cost of slight under-coverage, a finding visually confirmed in Figure~\ref{fig:darcy_ex}. Notably, the standard MC approach lacks functional coverage, as there is no global scalar metric to compare entire functions; therefore it can only provide pointwise uncertainty estimates. While calibration is effective in this setting, the computational overhead of sampling-based methods presents a significant bottleneck for deployment in high-resolution, complex PDEs.

\begin{figure}[!t]
    \centering
    \small
    \includegraphics[width=0.9\linewidth]{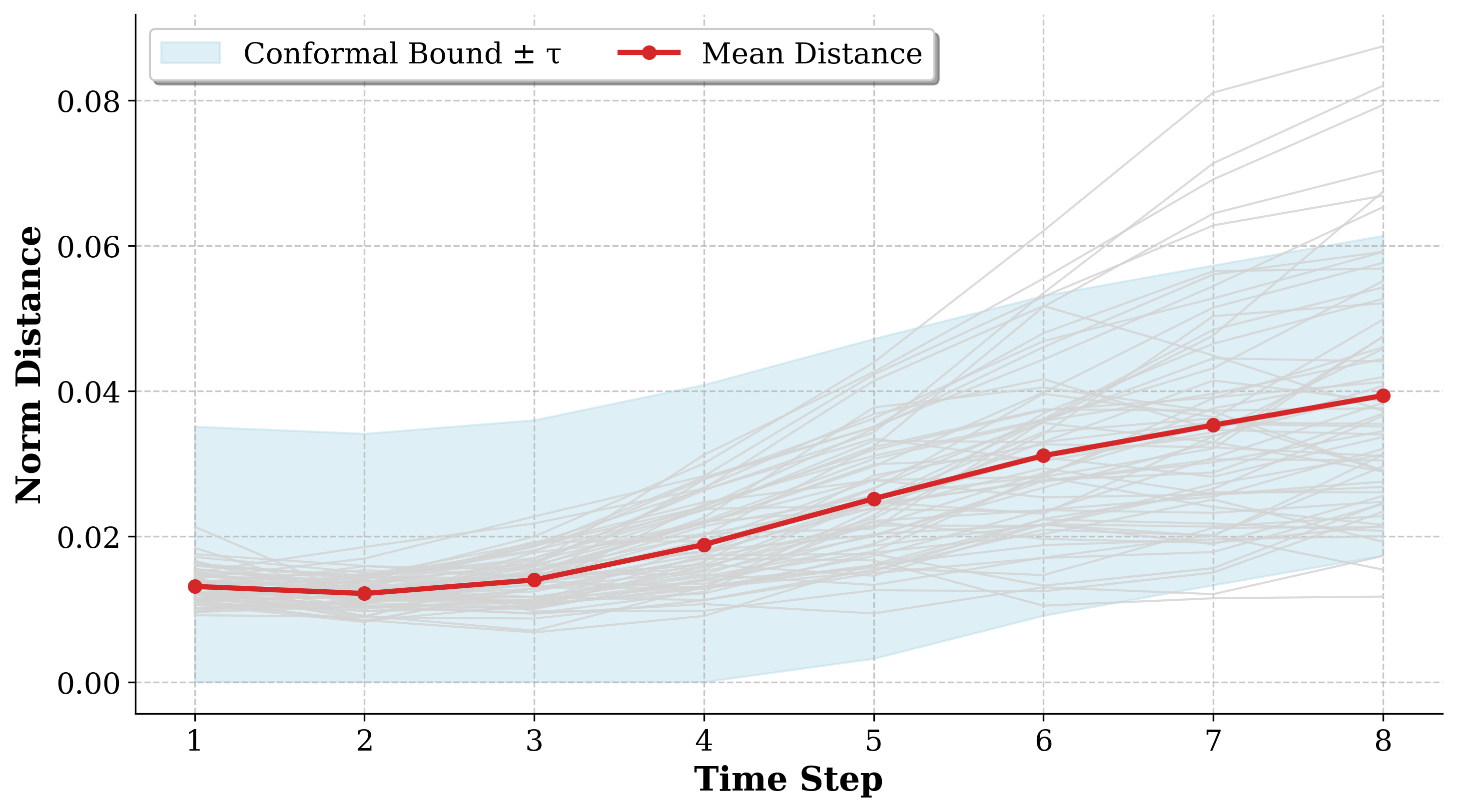}
    \caption{Navier-Stokes 2D: Bounded Ensemble. We focus on the mean-to-target distance (red), individual member distances (grey), and fixed conformal bound $\tau_\alpha$ (blue band) of a single instance from the evaluation set.}
    \label{fig:time_ensemble}
\end{figure}

\begin{table*}[!t]
\centering
\small
\renewcommand{\arraystretch}{1.0}
\setlength{\tabcolsep}{8pt}
\caption{Evolution of key metrics for the autoregressive forecast, evaluated at a significance level of $\alpha=0.1$ with a conformal threshold of $\tau=0.023$. The mean distance serves as the calibration error, the within $\tau_\alpha$ column gives a binary check of whether this error exceeds the threshold, and the ensemble spread measures forecast sharpness. Our CES score combines calibration error and sharpness to evaluate overall quality, while our IA metric quantifies the forecast's consensus.}
\label{tab:autoregressive_ces}
\begin{tabular}{lcccccccc}
\toprule
\textbf{Metric} & \textbf{$t=1$} & \textbf{$t=2$} & \textbf{$t=3$} & \textbf{$t=4$} & \textbf{$t=5$} & \textbf{$t=6$} & \textbf{$t=7$} & \textbf{$t=8$} \\
\midrule
\textbf{Within} $\tau_\alpha$ & Yes & Yes & Yes & Yes & Yes & \textbf{No} & \textbf{No} & \textbf{No} \\
\textbf{Mean Distance} & 0.0136 & 0.0126 & 0.0143 & 0.0170 & 0.0209 & 0.0260 & 0.0326 & 0.0391 \\
\textbf{Ensemble Spread} & 0.0132 & 0.0131 & 0.0152 & 0.0194 & 0.0252 & 0.0322 & 0.0402 & 0.0475 \\
\midrule
\textbf{IA (Ours)} & 98.0\% & 98.0\% & 92.2\% & 69.0\% & 36.3\% & 20.0\% & 13.8\% & 10.7\% \\
\textbf{CES (Ours)} & 0.0269 & 0.0257 & 0.0296 & 0.0364 & 0.0462 & 0.0583 & 0.0728 & 0.0867 \\
\textbf{CRPS} & 0.0082 & 0.0077 & 0.0096 & 0.0116 & 0.0142 & 0.0175 & 0.0216 & 0.0255 \\
\bottomrule
\end{tabular}
\end{table*}

\subsection{Poisson 2D: Adjusted Quantiles}
Next, we evaluate a more computationally efficient approach on the 2D Poisson equation using a neural operator that directly outputs both the solution and an error heuristic. As shown in Table~\ref{tab:poisson_coverage}, our method achieves pointwise coverage levels comparable to the risk-controlling PAC bounds formulated by \cite{ma2024calibrateduncertaintyquantificationoperator}. A key strength of our approach is its flexibility: since calibration is performed post hoc using a functional norm, the method accommodates both pointwise and functional uncertainty quantification, which the risk controlling bounds cannot do. Figure~\ref{fig:poisson_ex} illustrates that the learned error heuristic closely tracks the true error, though the resulting prediction intervals tend to be wider in many regions—a characteristic of models trained with the pinball loss. Although, the reliance on the pinball loss introduces a potential failure mode if it fails to accurately capture the true error distribution, a limitation noted in prior work.

\subsection{Navier-Stokes 2D: Bounded Ensemble Forecasting}
As the underlying dynamics become increasingly complex, learned error heuristics may fail, restricting us to quantifying distributional drift instead. For our final case, we consider stochastic autoregressive forecasting of the 2D Navier–Stokes vorticity field. As shown in Table~\ref{tab:autoregressive_ces}, forecast quality progressively degrades over time, evidenced by increasing CES and CRPS values and a sharp decline in IA across the forecast horizon. A key advantage of our framework is its ability to relate forecast degradation to the significance level. The \enquote{Within \(\tau_\alpha\)} column provides a binary indicator of calibration validity, flagging violations of the predefined safety threshold. Furthermore, because all components are expressed in the same relative units as \(\tau_\alpha\), the magnitudes of continuous diagnostic values can be directly interpreted with respect to the calibrated conformal radius. Finally, by jointly analyzing IA and CES, we observe that even when the ensemble mean remains within calibrated bounds, individual members may drift beyond a reasonable significance level, indicating loss of forecast reliability.

\subsection{Ablation Studies}
Finally, we conduct two ablation studies on the sensitivity to grid geometry and misspecification of resolution.

\begin{table}[!b]
\centering
\small
\renewcommand{\arraystretch}{1.0}
\setlength{\tabcolsep}{4pt} 
\caption{Comparison of calibrated thresholds $\tau_\alpha$ across grid geometries for $\alpha = 0.1$. The weighted norm yields more consistent $\tau_\alpha$ values with less variation under grid scheme.}
\begin{tabular}{lcc}
\toprule
\textbf{Grid Type} & \textbf{Relative Norm} & \textbf{Weighted Norm} \\
\midrule
Uniform & 0.02810 & 0.02810 \\
Clustered Center & 0.02866 & 0.02821 \\
Clustered Boundary & 0.03280 & 0.03079 \\
\midrule
Std. of Thresholds & 0.00257 & \textbf{0.00148} \\
Coeff. of Variation & 8.6\% & \textbf{5.1\%} \\
\bottomrule
\end{tabular}
\label{tab:tau_side_by_side}
\end{table}

\subsubsection{Conditional Coverage}
Assessing the effect of grid geometry on calibration is essential for demonstrating the necessity of our weighted norm. Using the 2D Poisson problem, we generate data on three distinct grid types: uniform, center-clustered, and boundary-clustered. For each grid, we compute \(\tau_\alpha\) using both the standard relative \(L^2\) norm and the weighted relative \(L^2\) norm. As shown in Table~\ref{tab:tau_side_by_side}, the weighted norm yields \(\tau_\alpha\) values with lower relative variation across grid types. In practice, neural operators often develop biases toward the discretization used during training. Our method mitigates this sensitivity by incorporating grid geometry (via cell areas), allowing the nonconformity score to more faithfully approximate the continuous-space error and thereby produce more reliable uncertainty sets.

\begin{table}[t!]
\centering
\small
\renewcommand{\arraystretch}{1.0}
\setlength{\tabcolsep}{8pt}
\caption{Coverage summary for Darcy 1D and Poisson 2D super-resolutions, calibrated at significance level $\alpha = 0.1$.}
\begin{tabular}{llcc}
\toprule
 & \textbf{Instance} & \textbf{Scalar} & \textbf{Coverage} \\
\midrule
\multirow{2}{*}{\textbf{Unadjusted}}
 & Darcy   & 0.00771 & 0.878 \\
 & Poisson & 0.09659 & 0.199 \\
\midrule
\multirow{2}{*}{\textbf{Adjusted}}
 & Darcy   & 0.00811 & \textbf{0.909} \\
 & Poisson & 0.10767 & \textbf{0.749} \\
\bottomrule
\end{tabular}
\label{tab:darcy_poisson_coverage}
\end{table}

\subsubsection{Super-Resolution Coverage}\label{sec:adj}
Next, we evaluate the super-resolution task introduced in the Theoretical Foundation section. Table~\ref{tab:darcy_poisson_coverage} shows that directly applying a conformal threshold \(\tau_\alpha\) calibrated at a low resolution—chosen to be slightly above the training resolution to exploit the observed log-linear convergence—results in substantial undercoverage at higher resolutions, with the Poisson case falling nearly 70\% below the target level. After adjustment, coverage for the Darcy case is fully restored to the desired 90\%, while the Poisson case, though still imperfect, improves markedly from 19.9\% to 74.9\%. These results indicate that the proposed adjustment substantially improves reliability, but does not fully recover target coverage under large resolution shifts, suggesting that the heuristic fails to capture the full complexity of the underlying error behavior. Furthermore, estimating the regression slope requires careful selection of calibration points, and thus relies on prior knowledge of the resolution-dependent behavior of \(\tau_\alpha(d)\).

\section{Conclusion}
This work extends split CP to function spaces using a quadrature-weighted norm, enabling calibrated uncertainty quantification for neural operators across varying grid resolutions. The proposed method improves stability under discretization shifts, achieves near target coverage in super-resolution tasks via a heuristic model of \(\tau_\alpha\), and enables reliability diagnostics in autoregressive forecasting through CES and IA. Empirical results demonstrate improved functional and pointwise coverage compared to baseline approaches, with particular gains under resolution mismatch and distributional drift.

\bibliography{ref}

\begin{thebibliography}{37}
\providecommand{\natexlab}[1]{#1}

\bibitem[{Akhare, Luo, and Wang(2023)}]{akhare2023diffhybriduquncertaintyquantificationdifferentiable}
Akhare, D.; Luo, T.; and Wang, J.-X. 2023.
\newblock DiffHybrid-{UQ}: Uncertainty Quantification for Differentiable Hybrid Neural Modeling.
\newblock arXiv:2401.00161.

\bibitem[{Baheri and Shahbazi(2025)}]{baheri2025multi}
Baheri, A.; and Shahbazi, M.~A. 2025.
\newblock Multi-Scale Conformal Prediction: A Theoretical Framework with Coverage Guarantees.
\newblock \emph{arXiv preprint arXiv:2502.05565}.

\bibitem[{Bahmani et~al.(2025)Bahmani, Goswami, Kevrekidis, and Shields}]{BAHMANI2025118113}
Bahmani, B.; Goswami, S.; Kevrekidis, I.~G.; and Shields, M.~D. 2025.
\newblock A resolution independent neural operator.
\newblock \emph{Computer Methods in Applied Mechanics and Engineering}, 444: 118113.

\bibitem[{Barber et~al.(2023)Barber, Cand{\`e}s, Ramdas, and Tibshirani}]{barber2023conformalpredictionexchangeability}
Barber, R.~F.; Cand{\`e}s, E.~J.; Ramdas, A.; and Tibshirani, R.~J. 2023.
\newblock {Conformal prediction beyond exchangeability}.
\newblock \emph{The Annals of Statistics}, 51(2): 816 -- 845.

\bibitem[{Brofos, Shu, and Lederman(2019)}]{brofos2019bias}
Brofos, J.; Shu, R.; and Lederman, R.~R. 2019.
\newblock A bias-variance decomposition for Bayesian deep learning.
\newblock In \emph{NeurIPS 2019 Workshop on Bayesian Deep Learning}.

\bibitem[{B{\"u}lte et~al.(2025)B{\"u}lte, Horat, Quinting, and Lerch}]{bulte2025uncertainty}
B{\"u}lte, C.; Horat, N.; Quinting, J.; and Lerch, S. 2025.
\newblock Uncertainty quantification for data-driven weather models.
\newblock \emph{Artificial Intelligence for the Earth Systems}.

\bibitem[{Cleaveland et~al.(2024)Cleaveland, Lee, Pappas, and Lindemann}]{cleaveland2024conformal}
Cleaveland, M.; Lee, I.; Pappas, G.~J.; and Lindemann, L. 2024.
\newblock Conformal prediction regions for time series using linear complementarity programming.
\newblock In \emph{Proceedings of the AAAI Conference on Artificial Intelligence}, volume~38, 20984--20992.

\bibitem[{Durasov et~al.(2021)Durasov, Bagautdinov, Baque, and Fua}]{durasov2021masksembles}
Durasov, N.; Bagautdinov, T.; Baque, P.; and Fua, P. 2021.
\newblock Masksembles for uncertainty estimation.
\newblock In \emph{Proceedings of the IEEE/CVF Conference on Computer Vision and Pattern Recognition}, 13539--13548.

\bibitem[{Folgoc et~al.(2021)Folgoc, Baltatzis, Desai, Devaraj, Ellis, Manzanera, Nair, Qiu, Schnabel, and Glocker}]{folgoc2021mc}
Folgoc, L.~L.; Baltatzis, V.; Desai, S.; Devaraj, A.; Ellis, S.; Manzanera, O. E.~M.; Nair, A.; Qiu, H.; Schnabel, J.; and Glocker, B. 2021.
\newblock Is MC dropout bayesian?
\newblock \emph{arXiv preprint arXiv:2110.04286}.

\bibitem[{Fontana, Zeni, and Vantini(2023)}]{10.3150/21-BEJ1447}
Fontana, M.; Zeni, G.; and Vantini, S. 2023.
\newblock {Conformal prediction: A unified review of theory and new challenges}.
\newblock \emph{Bernoulli}, 29(1): 1 -- 23.

\bibitem[{Furuya, Taniguchi, and Okuda(2024)}]{furuya2024quantitative}
Furuya, T.; Taniguchi, K.; and Okuda, S. 2024.
\newblock Quantitative Approximation for Neural Operators in Nonlinear Parabolic Equations.
\newblock \emph{arXiv preprint arXiv:2410.02151}.

\bibitem[{Ghosh et~al.(2023)Ghosh, Shi, Belkhouja, Yan, Doppa, and Jones}]{ghosh2023probabilisticallyrobustconformalprediction}
Ghosh, S.; Shi, Y.; Belkhouja, T.; Yan, Y.; Doppa, J.; and Jones, B. 2023.
\newblock Probabilistically robust conformal prediction.
\newblock In Evans, R.~J.; and Shpitser, I., eds., \emph{Proceedings of the Thirty-Ninth Conference on Uncertainty in Artificial Intelligence}, volume 216 of \emph{Proceedings of Machine Learning Research}, 681--690. PMLR.

\bibitem[{Goan and Fookes(2020)}]{Goan2020}
Goan, E.; and Fookes, C. 2020.
\newblock \emph{Bayesian Neural Networks: An Introduction and Survey}, 45--87.
\newblock Cham: Springer International Publishing.
\newblock ISBN 978-3-030-42553-1.

\bibitem[{Guo et~al.(2024)Guo, Wu, Wang, Zhou, and Zhou}]{guo2023ibuqinformationbottleneckbased}
Guo, L.; Wu, H.; Wang, Y.; Zhou, W.; and Zhou, T. 2024.
\newblock IB-UQ: Information bottleneck based uncertainty quantification for neural function regression and neural operator learning.
\newblock \emph{Journal of Computational Physics}, 510: 113089.

\bibitem[{Hastie et~al.(2009)Hastie, Tibshirani, Friedman, and Friedman}]{hastie2009elements}
Hastie, T.; Tibshirani, R.; Friedman, J.~H.; and Friedman, J.~H. 2009.
\newblock \emph{The elements of statistical learning: data mining, inference, and prediction}, volume~2.
\newblock Springer.

\bibitem[{Haussler and Warmuth(2018)}]{haussler2018probably}
Haussler, D.; and Warmuth, M. 2018.
\newblock The probably approximately correct (PAC) and other learning models.
\newblock \emph{The Mathematics of Generalization}, 17--36.

\bibitem[{Hicken and Zingg(2013)}]{hicken2013summation}
Hicken, J.~E.; and Zingg, D.~W. 2013.
\newblock Summation-by-parts operators and high-order quadrature.
\newblock \emph{Journal of Computational and Applied Mathematics}, 237(1): 111--125.

\bibitem[{Huang, Shi, and Suykens(2013)}]{huang2013support}
Huang, X.; Shi, L.; and Suykens, J.~A. 2013.
\newblock Support vector machine classifier with pinball loss.
\newblock \emph{IEEE transactions on pattern analysis and machine intelligence}, 36(5): 984--997.

\bibitem[{Katende(2025)}]{katende2025stability}
Katende, R. 2025.
\newblock Stability Analysis of Physics-Informed Neural Networks via Variational Coercivity, Perturbation Bounds, and Concentration Estimates.
\newblock \emph{arXiv preprint arXiv:2506.13554}.

\bibitem[{Kiranyaz et~al.(2021)Kiranyaz, Malik, Abdallah, Ince, Iosifidis, and Gabbouj}]{KIRANYAZ2021294}
Kiranyaz, S.; Malik, J.; Abdallah, H.~B.; Ince, T.; Iosifidis, A.; and Gabbouj, M. 2021.
\newblock Self-organized Operational Neural Networks with Generative Neurons.
\newblock \emph{Neural Networks}, 140: 294--308.

\bibitem[{Kull et~al.(2019)Kull, Perello~Nieto, K{\"a}ngsepp, Silva~Filho, Song, and Flach}]{kull2019beyond}
Kull, M.; Perello~Nieto, M.; K{\"a}ngsepp, M.; Silva~Filho, T.; Song, H.; and Flach, P. 2019.
\newblock Beyond temperature scaling: Obtaining well-calibrated multi-class probabilities with dirichlet calibration.
\newblock \emph{Advances in neural information processing systems}, 32.

\bibitem[{Lanthaler, Stuart, and Trautner(2024)}]{lanthaler2024discretizationerrorfourierneural}
Lanthaler, S.; Stuart, A.~M.; and Trautner, M. 2024.
\newblock Discretization Error of {Fourier} Neural Operators.
\newblock arXiv:2405.02221.

\bibitem[{Lara~Benitez et~al.(2024)Lara~Benitez, Furuya, Faucher, Kratsios, Tricoche, and de~Hoop}]{Lara_Benitez_2024}
Lara~Benitez, J.~A.; Furuya, T.; Faucher, F.; Kratsios, A.; Tricoche, X.; and de~Hoop, M.~V. 2024.
\newblock Out-of-distributional risk bounds for neural operators with applications to the Helmholtz equation.
\newblock \emph{Journal of Computational Physics}, 513: 113168.

\bibitem[{Le and Dik(2024)}]{le2024mathematical}
Le, V.-A.; and Dik, M. 2024.
\newblock A mathematical analysis of neural operator behaviors.
\newblock \emph{arXiv preprint arXiv:2410.21481}.

\bibitem[{Li et~al.(2021)Li, Kovachki, Azizzadenesheli, liu, Bhattacharya, Stuart, and Anandkumar}]{li2021fourierneuraloperatorparametric}
Li, Z.; Kovachki, N.~B.; Azizzadenesheli, K.; liu, B.; Bhattacharya, K.; Stuart, A.; and Anandkumar, A. 2021.
\newblock Fourier Neural Operator for Parametric Partial Differential Equations.
\newblock In \emph{International Conference on Learning Representations}.

\bibitem[{Liu and Tang(2025)}]{Liu_2025_CVPR}
Liu, X.; and Tang, H. 2025.
\newblock DiffFNO: Diffusion Fourier Neural Operator.
\newblock In \emph{Proceedings of the Computer Vision and Pattern Recognition Conference (CVPR)}, 150--160.

\bibitem[{Lu et~al.(2021)Lu, Jin, Pang, Zhang, and Karniadakis}]{Lu_2021}
Lu, L.; Jin, P.; Pang, G.; Zhang, Z.; and Karniadakis, G.~E. 2021.
\newblock Learning nonlinear operators via DeepONet based on the universal approximation theorem of operators.
\newblock \emph{Nature Machine Intelligence}, 3(3): 218–229.

\bibitem[{Ma et~al.(2024)Ma, Pitt, Azizzadenesheli, and Anandkumar}]{ma2024calibrateduncertaintyquantificationoperator}
Ma, Z.; Pitt, D.; Azizzadenesheli, K.; and Anandkumar, A. 2024.
\newblock Calibrated Uncertainty Quantification for Operator Learning via Conformal Prediction.
\newblock \emph{Transactions on Machine Learning Research}.

\bibitem[{Pic et~al.(2023)Pic, Dombry, Naveau, and Taillardat}]{Pic_2023}
Pic, R.; Dombry, C.; Naveau, P.; and Taillardat, M. 2023.
\newblock Distributional regression and its evaluation with the CRPS: Bounds and convergence of the minimax risk.
\newblock \emph{International Journal of Forecasting}, 39(4): 1564–1572.

\bibitem[{Qian et~al.(2023)Qian, Zhang, Zhao, Zheng, and Yu}]{qian2023uncertainty}
Qian, W.; Zhang, D.; Zhao, Y.; Zheng, K.; and Yu, J.~J. 2023.
\newblock Uncertainty quantification for traffic forecasting: A unified approach.
\newblock In \emph{2023 IEEE 39th International Conference on Data Engineering (ICDE)}, 992--1004. IEEE.

\bibitem[{Rahaman et~al.(2021)}]{rahaman2021uncertainty}
Rahaman, R.; et~al. 2021.
\newblock Uncertainty quantification and deep ensembles.
\newblock \emph{Advances in neural information processing systems}, 34: 20063--20075.

\bibitem[{Ranganath et~al.(2016)Ranganath, Tran, Altosaar, and Blei}]{NIPS2016_d947bf06}
Ranganath, R.; Tran, D.; Altosaar, J.; and Blei, D. 2016.
\newblock Operator Variational Inference.
\newblock In Lee, D.; Sugiyama, M.; Luxburg, U.; Guyon, I.; and Garnett, R., eds., \emph{Advances in Neural Information Processing Systems}, volume~29. Curran Associates, Inc.

\bibitem[{Schulz, Speekenbrink, and Krause(2018)}]{SCHULZ20181}
Schulz, E.; Speekenbrink, M.; and Krause, A. 2018.
\newblock A tutorial on Gaussian process regression: Modelling, exploring, and exploiting functions.
\newblock \emph{Journal of Mathematical Psychology}, 85: 1--16.

\bibitem[{Scoccimarro(1998)}]{scoccimarro1998transients}
Scoccimarro, R. 1998.
\newblock Transients from initial conditions: a perturbative analysis.
\newblock \emph{Monthly Notices of the Royal Astronomical Society}, 299(4): 1097--1118.

\bibitem[{Shapiro(2003)}]{shapiro2003monte}
Shapiro, A. 2003.
\newblock Monte Carlo sampling methods.
\newblock \emph{Handbooks in operations research and management science}, 10: 353--425.

\bibitem[{Tran et~al.(2020)Tran, Dwelle, Sargsyan, Ivanov, and Kim}]{tran2020novel}
Tran, V.~N.; Dwelle, M.~C.; Sargsyan, K.; Ivanov, V.~Y.; and Kim, J. 2020.
\newblock A novel modeling framework for computationally efficient and accurate real-time ensemble flood forecasting with uncertainty quantification.
\newblock \emph{Water Resources Research}, 56(3): e2019WR025727.

\bibitem[{Zou, Meng, and Karniadakis(2025)}]{zou2023uncertaintyquantificationnoisyinputsoutputs}
Zou, Z.; Meng, X.; and Karniadakis, G.~E. 2025.
\newblock Uncertainty quantification for noisy inputs–outputs in physics-informed neural networks and neural operators.
\newblock \emph{Computer Methods in Applied Mechanics and Engineering}, 433: 117479.

\end{thebibliography}

\appendix
\onecolumn

\section{Code Availability}
Our repository is available under the Apache License version 2.0. Our checkpoints and output (including plots, tables, and animations) are too large to be included in the repository, as the total size is approximately 5 GB for the checkpoints and 39 GB for the output. However, they can be obtained upon request from the corresponding author.

\section{Detailed Proof of Theorem 1}

\begin{proof}
Let $\{(f_i, u_i)\}_{i=1}^{n+1}$ be an i.i.d. sample from the data-generating distribution on $\mathcal{X} \times \mathcal{Y}$. Let $\hat{u}_i = \mathcal{G}_\theta(f_i)$ be the model prediction for input $f_i$. Our goal is to construct a prediction set for $u_{n+1}$ in the function space $\mathcal{Y}$ with a coverage guarantee. The proof proceeds in five steps.

\noindent\textbf{Step 1: Calibration in Discretized Space.}
    We define a nonconformity score on the discretized space $\mathbb{R}^d$ using the projection operator $P: \mathcal{Y} \rightarrow \mathbb{R}^d$ and the quadrature-weighted norm $||\cdot||_{w,2,d}$. For each of the first $n$ samples in the calibration set, we compute the score:
    \begin{equation}
        s_i = ||P(\hat{u}_i) - P(u_i)||_{w,2,d}.
    \end{equation}
    We then compute the conformal threshold $\tau_\alpha$ as the $\lceil(1-\alpha)(n+1)\rceil / (n+1)$-th empirical quantile of the scores $\{s_1, \dots, s_n\}$. By the standard theory of split conformal prediction, the score for the test point $s_{n+1}$ satisfies:
    \begin{equation}
        \mathbb{P}(s_{n+1} \le \tau_\alpha) \ge 1 - \alpha.
    \end{equation}
    This provides a finite-sample, distribution-free guarantee in the discretized space.

\noindent\textbf{Step 2: Relating Discrete and Continuous Nonconformity.}
    The core of our functional guarantee relies on relating the discrete score $s_i$ to a true nonconformity score in the continuous function space, $s_i^{\text{cont}} = ||\hat{u}_i - u_i||_\mathcal{Y}$. This relationship is formalized by the bilipschitz assumption (Assumption 1), which states that there exist constants $c_1, c_2 > 0$ such that for any $u, v \in \mathcal{Y}$:
    \begin{equation}
        c_1 ||u - v||_\mathcal{Y} \le ||P(u) - P(v)||_{w,2,d} \le c_2 ||u - v||_\mathcal{Y}.
    \end{equation}
    Applying this to our nonconformity scores, we have:
    \begin{equation}
        c_1 s_i^{\text{cont}} \le s_i \le c_2 s_i^{\text{cont}} \quad \forall i=1, \dots, n+1.
    \end{equation}

\noindent\textbf{Step 3: Deriving the Functional Coverage Guarantee.}
    We can now translate the coverage guarantee from the discrete space to the continuous space. The event $s_{n+1} \le \tau_\alpha$ is the basis of our $1-\alpha$ guarantee. From the bilipschitz inequality, this event implies a condition on the continuous score: if $s_{n+1} \le \tau_\alpha$, then $c_1 s_{n+1}^{\text{cont}} \le s_{n+1} \le \tau_\alpha$, which implies $s_{n+1}^{\text{cont}} \le \tau_\alpha / c_1$.
    Because the event $\{s_{n+1} \le \tau_\alpha\}$ implies the event $\{s_{n+1}^{\text{cont}} \le \tau_\alpha / c_1\}$, the probability of the latter must be at least as large as the probability of the former:
    \begin{equation}
        \mathbb{P}(s_{n+1}^{\text{cont}} \le \tau_\alpha / c_1) \ge \mathbb{P}(s_{n+1} \le \tau_\alpha) \ge 1 - \alpha.
    \end{equation}

\noindent\textbf{Step 4: Defining the Functional Prediction Set.}
    This result allows us to define a prediction set directly in the function space $\mathcal{Y}$. We define the set $\Gamma_\alpha^{\text{func}}(f_{n+1})$ as a ball in $\mathcal{Y}$ with radius $\tau_\alpha / c_1$:
    \begin{equation}
        \Gamma_\alpha^{\text{func}}(f_{n+1}) = \{v \in \mathcal{Y} : ||\hat{u}_{n+1} - v||_\mathcal{Y} \le \tau_\alpha / c_1\}.
    \end{equation}
    From Step 3, the probability that the true function $u_{n+1}$ lies in this set is:
    \begin{equation}
        \mathbb{P}(u_{n+1} \in \Gamma_\alpha^{\text{func}}(f_{n+1})) = \mathbb{P}(s_{n+1}^{\text{cont}} \le \tau_\alpha / c_1) \ge 1 - \alpha.
    \end{equation}
    This provides a valid finite-sample coverage guarantee for a prediction set in the function space $\mathcal{Y}$.

\noindent\textbf{Step 5: The Asymptotic Argument.}
    The "asymptotic" nature of the guarantee refers to the fact that the radius of the guaranteed set, $\tau_\alpha / c_1$, converges to the calibrated discrete radius $\tau_\alpha$ as the discretization becomes arbitrarily fine. As the discretization is refined (i.e., $d \rightarrow \infty$), the quadrature-weighted norm converges to the continuous $L^2$ norm, and the projection $P$ becomes a near-isometry for the class of functions of interest. In this limit, the distortion constant $c_1 \to 1$. Thus, for sufficiently fine discretizations, the guaranteed function-space ball has a radius that approaches the empirically calibrated threshold $\tau_\alpha$.
\end{proof}

\begin{remark}[Asymptotic Convergence of the Prediction Set Radius.]
The practical utility of the guarantee depends on the constant $c_1$, which quantifies the geometric distortion introduced by the discretization map $P$. The "asymptotic" aspect of our framework relates to the behavior of this constant. As the discretization is refined (i.e., $d \rightarrow \infty$), the quadrature-weighted norm converges to the continuous $L^2$ norm (as shown in Theorem 4), and the projection $P$ becomes a near-isometry for the class of functions of interest. In this limit, the distortion constant $c_1 \to 1$. Consequently, for sufficiently fine discretizations, the radius of the guaranteed function-space ball, $\tau_\alpha / c_1$, approaches the empirically calibrated and directly computable threshold $\tau_\alpha$. This ensures that our functional prediction set is not only theoretically valid but also practically useful, as its size is determined by a well-behaved, empirical quantity.
\end{remark}

\section{Detailed Proof of Theorem 2}
\begin{proof}
Let $\{(f_i, u_i)\}_{i=1}^{n}$ be an i.i.d. calibration set drawn from $\mathbb{P}_{\mathrm{cal}}$, and let $\hat{u}_i = \mathcal{G}_\theta(f_i)$ denote model predictions. We aim to characterize the functional coverage when the same conformal threshold $\tau_\alpha$ is applied to data at forecast time $t$ from a potentially different distribution $\mathbb{P}_t$. The proof follows three steps.

\noindent\textbf{Step 1: Coverage under Distribution Shift.}
By the standard split CP guarantee, the discrete scores $s_i = \|P_d(\hat{u}_i) - P_d(u_i)\|_{w,2,d}$ satisfy
$
\mathbb{P}_{\mathrm{cal}}\big(s \le \tau_\alpha\big) \ge 1-\alpha.
$
When this threshold is applied to samples drawn from $\mathbb{P}_t$, the drift bound of \cite{barber2023conformalpredictionexchangeability} implies
\begin{equation}
\mathbb{P}_{t}\big(s \le \tau_\alpha\big) \ge 1-\alpha - d_{\mathrm{TV}}(\mathbb{P}_{\mathrm{cal}}, \mathbb{P}_t),
\end{equation}
where $d_{\mathrm{TV}}$ is the total variation distance between $\mathbb{P}_{\mathrm{cal}}$ and $\mathbb{P}_t$.

\noindent\textbf{Step 2: Lifting to the Function Space.}
From Assumption~\ref{assum:bilipschitz}, we have $c_1 \|\hat{u}_t - u_t\|_{\mathcal{Y}} \le s$. Thus, the event $s \le \tau_\alpha$ implies
\begin{equation}
\|\hat{u}_t - u_t\|_{\mathcal{Y}} \le \tau_\alpha / c_1.
\end{equation}

\noindent\textbf{Step 3: Functional Prediction Set Coverage.}
Define the functional prediction set 
$
\Gamma_{\alpha}^{\text{func}}(f_t) = \{ v \in \mathcal{Y} : \| \hat{u}_t - v \|_{\mathcal{Y}} \le \tau_\alpha / c_1 \}.
$
By Step 2, $s \le \tau_\alpha$ implies $u_t \in \Gamma_{\alpha}^{\text{func}}(f_t)$. Combining this implication with the drift-adjusted discrete guarantee from Step 1
\begin{equation}
\mathbb{P}_t\left(u_t \in \Gamma_{\alpha}^{\text{func}}(f_t)\right) \ge 1-\alpha - d_{\mathrm{TV}}(\mathbb{P}_{\mathrm{cal}}, \mathbb{P}_t),
\end{equation}
which establishes the stated result.
\end{proof}

\section{Additional Theoretical Results}
Now we present three additional theoretical results, which further illuminate the connection between discrete and functional coverage. The first result Lemma~\ref{thrm:proj_stabiltiy} formalizes a stability property of the projection operator, ensuring that small discretization errors cannot drastically change the geometry of the output space. The second result Theorem 3 incorporates this stability into a PDE setting, where the solution operator is assumed to be stable with respect to perturbations in boundary conditions or source terms. The third result Theorem shows the convergence of the weighted norm to the continuous $L^2$ norm.

\begin{lemma}[Projection Stability]\label{thrm:proj_stabiltiy}
Let $\mathcal{Y}$ be a normed function space equipped with $\|\cdot\|_{\mathcal{Y}}$, and let $P: \mathcal{Y} \rightarrow \mathbb{R}^d$ be a discretization (or projection) map. Suppose that there exist constants $c_1, c_2 > 0$ such that for all $u, v \in \mathcal{Y}$,
\begin{equation}
c_1 \|u-v\|_{\mathcal{Y}} \leq \|P(u)-P(v)\|_{\mathbb{R}^d} \leq c_2 \|u-v\|_{\mathcal{Y}}.
\end{equation}
Then $P$ is bilipschitz on $\mathcal{Y}$, up to the constants $c_1$ and $c_2$. Consequently, any ball in $\mathcal{Y}$ is mapped to a comparable ball in $\mathbb{R}^d$, and vice versa.
\end{lemma}

\begin{proof}
The assumed double inequality implies that $P$ preserves distances in $\mathcal{Y}$ up to multiplicative factors $c_1$ and $c_2$. Specifically,
\begin{equation}
c_1\|u-v\| y \leq\|P(u)-P(v)\|_{\mathbb{R}^d} \leq c_2\|u-v\|_y
\end{equation}
Hence, if $\|u-v\| y \leq \epsilon$, it follows that $\|P(u)-P(v)\|_{\mathbb{R}^d} \leq c_2 \epsilon$.
Conversely, $\|P(u)-P(v)\|_{\mathbb{R}^d} \leq \delta$ implies $\|u-v\| y \leq \frac{\delta}{c_1}$.
Thus, neighborhoods in $\mathcal{Y}$ map to neighborhoods in $\mathbb{R}^d$ with at most a constant change in radius, establishing bilipschitz continuity of $P$.
\end{proof}

\begin{theorem}[Conformal Coverage for PDE Solutions Under Operator Stability]\label{thr:2}
Let $\mathcal{A}: \mathcal{X} \rightarrow \mathcal{Y}$ be a forward operator that maps an input function (e.g., initial/boundary data, source terms) $f \in \mathcal{X}$ to the corresponding PDE solution $u=\mathcal{A}(f) \in \mathcal{Y}$. Assume $\mathcal{A}$ is Lipschitz stable with constant $L>0$; that is, for any $f_1, f_2 \in \mathcal{X}$,
$$
\left\|\mathcal{A}\left(f_1\right)-\mathcal{A}\left(f_2\right)\right\|_{\mathcal{Y}} \leq L\left\|f_1-f_2\right\|_{\mathcal{X}}
$$
Suppose further that the discretization map $P: \mathcal{Y} \rightarrow \mathbb{R}^d$ satisfies the bilipschitz property in the Lemma above, with constants $c_1 \leq c_2$. Let $\mathcal{G}_\theta$ be a neural-operator-based approximation of $\mathcal{A}$, and let $\widehat{\Gamma}_\alpha(\cdot) \subseteq \mathbb{R}^d$ be a (discretized) conformal set calibrated to achieve
$$
\mathbb{P}\left[P\left(u_{n+1}\right) \notin \widehat{\Gamma}_\alpha\left(f_{n+1}\right)\right] \leq \alpha
$$
under i.i.d. draws $\left\{\left(f_i, u_i\right)\right\} \sim \mathcal{D}$. Define the functional conformal set
$$
\Gamma_\alpha(f)=\left\{v \in \mathcal{Y}: P(v) \in \widehat{\Gamma}_\alpha(f)\right\}
$$
Then, for sufficiently fine discretization (i.e., adequately large $d$ ), there exists $\delta \geq 0$ such that
$$
\mathbb{P}\left[u_{n+1} \notin \Gamma_\alpha\left(f_{n+1}\right)\right] \leq \alpha+\delta
$$
where $u_{n+1}=\mathcal{A}\left(f_{n+1}\right)$. Moreover, $\delta \rightarrow 0$ as the mesh or basis in $P$ is refined and the training sample size $n$ grows.
\end{theorem}

\begin{proof}
he key is to formalize the role of the PDE operator $\mathcal{A}$ in ensuring that the conditions for Theorem 1 are met.

\noindent\textbf{Step 1: Defining the Relevant Function Subspace.}
    The set of all possible PDE solutions forms a specific subset of the larger function space $\mathcal{Y}$. We define this subset as the range of the solution operator, $\mathcal{Y}_{\mathcal{A}} = \{ \mathcal{A}(f) \mid f \in \mathcal{X} \}$. The Lipschitz stability of $\mathcal{A}$ implies that solutions in $\mathcal{Y}_{\mathcal{A}}$ possess a certain regularity (e.g., they lie within a ball in a Sobolev space). This regularity is critical because it makes the bilipschitz assumption on the discretization map $P$ plausible.

\noindent\textbf{Step 2: Applying the Bilipschitz Assumption on the Subspace.}
    We assume that the discretization map $P: \mathcal{Y} \rightarrow \mathbb{R}^d$ is bilipschitz specifically on the subspace of solutions $\mathcal{Y}_{\mathcal{A}}$. That is, there exist constants $c_1, c_2 > 0$ such that for any two solutions $u, v \in \mathcal{Y}_{\mathcal{A}}$:
    \begin{equation}
        c_1 ||u - v||_\mathcal{Y} \le ||P(u) - P(v)||_{w,2,d} \le c_2 ||u - v||_\mathcal{Y}.
    \end{equation}
    This is a more targeted assumption than requiring the property to hold over all of $\mathcal{Y}$. The stability of $\mathcal{A}$ makes this assumption reasonable for standard discretization methods (e.g., sufficiently fine grids for smooth solutions).

\noindent\textbf{Step 3: Invoking Theorem 3.}
    The calibration data consists of pairs $(f_i, u_i)$ where $u_i = \mathcal{A}(f_i)$, so all true solutions $u_i$ belong to $\mathcal{Y}_{\mathcal{A}}$. The neural operator $\mathcal{G}_\theta$ is trained to approximate $\mathcal{A}$, so its predictions $\hat{u}_i = \mathcal{G}_\theta(f_i)$ are also expected to lie in or near this subspace.
    Since the bilipschitz condition holds for the relevant functions (the true solutions and their approximations), all conditions for Theorem 3 are met. We calibrate a threshold $\tau_\alpha$ in the discrete space as per Step 1 of the proof of Theorem 3. Then, following Steps 2-4 of that proof, we can construct a functional prediction set:
    \begin{equation}
        \Gamma_\alpha^{\text{func}}(f_{n+1}) = \{v \in \mathcal{Y} : ||\hat{u}_{n+1} - v||_\mathcal{Y} \le \tau_\alpha / c_1\},
    \end{equation}
    which is guaranteed to have coverage of at least $1-\alpha$:
    \begin{equation}
        \mathbb{P}(u_{n+1} \in \Gamma_\alpha^{\text{func}}(f_{n+1})) \ge 1 - \alpha,
    \end{equation}
    where $u_{n+1} = \mathcal{A}(f_{n+1})$. The asymptotic nature relates to $c_1 \to 1$ as the discretization is refined. This completes the proof.
\end{proof}

\begin{remark}
The Lipschitz stability assumption on $\mathcal{A}$ arises naturally in many PDE problems with wellposedness guarantees, where small perturbations in boundary conditions or source terms lead to proportionally bounded changes in the solution field. The theorem then assures that, for stable PDE operators and suitably refined discretization, the functional conformal predictor $\Gamma_\alpha\left(f_{n+1}\right)$ offers robust finite-sample coverage in the infinite-dimensional setting.
\end{remark}

\begin{theorem}[Convergence to the Continuous $L^2$ Norm]
Let $f$ be a Riemann integrable function on a bounded domain $\Omega \subset \mathbb{R}^2$. Let $P_d$ be a sequence of partitions of $\Omega$ into $d$ cells, indexed by $d \in \mathbb{N}$, such that the norm of the partition, $\|P_d\| = \max_{i=1,\dots,d} (\text{diam}(C_i))$, tends to zero as $d \to \infty$. Let $\|f\|_{w,2,d}^2$ be the squared quadrature-weighted norm computed on the partition $P_d$. Then, the sequence of discrete norms converges to the squared continuous $L^2$ norm:
\begin{equation}
\lim_{d \to \infty} \|f\|_{w,2,d}^2 = \|f\|_{L^2(\Omega)}^2
\end{equation}
\end{theorem}

\begin{proof}
\noindent We establish the equivalence in four steps, following the definition of the Riemann integral.

\noindent\textbf{Step 1: Formal Definitions.}
The squared continuous $L^2$ norm of a function $f: \Omega \to \mathbb{R}$ is given by the definite integral:
\begin{equation}
\|f\|_{L^2(\Omega)}^2 = \int_{\Omega} f(x)^2 \,dx
\end{equation}
A partition $P_d$ of the domain $\Omega$ consists of a set of $d$ non-overlapping cells $\{C_1, C_2, \dots, C_d\}$ such that $\cup_{i=1}^d C_i = \Omega$. For each cell $C_i$, we denote its area by $w_i = \text{Area}(C_i)$ and choose a sample point $x_i^* \in C_i$.
The squared quadrature-weighted discrete norm is defined on this partition as:
\begin{equation}
\|f\|_{w,2,d}^2 = \sum_{i=1}^d w_i f(x_i^*)^2
\end{equation}

\noindent\textbf{Step 2: Identification as a Riemann Sum.}
Let the function $g(x) = f(x)^2$. Since $f$ is Riemann integrable, $g$ is also Riemann integrable on $\Omega$. The expression for the squared quadrature-weighted norm,
\begin{equation}
\sum_{i=1}^d f(x_i^*)^2 w_i = \sum_{i=1}^d g(x_i^*) w_i
\end{equation}
is precisely the definition of a Riemann sum for the function $g(x)$ over the domain $\Omega$ with respect to the partition $P_d$ and the sample points $\{x_i^*\}$.

\noindent\textbf{Step 3: Convergence via the Definition of the Riemann Integral.}
The definite integral of a function $g$ over a domain $\Omega$ is defined as the limit of its Riemann sums as the norm of the partition (i.e., the maximum diameter of any cell in the partition) approaches zero. Formally,
\begin{equation}
\int_{\Omega} g(x) \,dx = \lim_{\|P_d\| \to 0} \sum_{i=1}^d g(x_i^*) w_i
\end{equation}
This limit exists and is independent of the choice of sample points $x_i^*$ because $g$ is Riemann integrable. As we refine our grid such that $d \to \infty$ and $\|P_d\| \to 0$, our discrete norm calculation becomes an increasingly accurate approximation of this integral.

\noindent\textbf{Step 4: Conclusion of Equivalence.}
By substituting $g(x) = f(x)^2$ and applying the definition of the Riemann integral from Step 3, we directly connect the limit of the discrete norm to the continuous norm:
\begin{equation}
\lim_{d \to \infty} \|f\|_{w,2,d}^2 = \lim_{\|P_d\| \to 0} \sum_{i=1}^d f(x_i^*)^2 w_i = \int_{\Omega} f(x)^2 \,dx = \|f\|_{L^2(\Omega)}^2
\end{equation}
This convergence proves that the quadrature-weighted discrete norm is fundamentally equivalent to the continuous $L^2$ norm in the limit. This property justifies its use as a reliable and geometrically sound nonconformity score for function-space conformal prediction.
\end{proof}

\section{PDEs \& Data Generation}

\subsection{Darcy Flow}
For the first case we consider a one-dimensional variant of Darcy flow governed by the steady-state elliptic problem
\begin{equation}
-\frac{d}{d x}\left(k(x) \frac{d u}{d x}\right)=0 \quad \text { on } x \in[0,1]
\end{equation}
subject to Dirichlet boundary conditions $u(0)=0$ and $u(1)=1$. The coefficient $k(x)$ represents a permeability field that varies in space. In order to emulate heterogeneous media, we draw $k(x)$ from a random field with mild smoothness, ensuring that it remains strictly positive. The essential objective is to approximate the mapping $k(\cdot) \mapsto u(\cdot)$ and then provide set-valued predictions $\Gamma_\alpha(\cdot)$ with a guaranteed coverage level $1-\alpha$.

For data generation, we discretize the domain $[0,1]$ into $1024$ uniform points and construct each random permeability $k(x)$ by summing a few random Fourier modes. This yields fields with different oscillatory patterns, constrained so that $k(x)$ remains between 0.01 and 10. For each sample, we solve the linear system that arises from the finite-difference approximation of the Darcy equation, thereby obtaining a ground-truth solution $u(x)$. We generated 100,000 such permeability-solution pairs and subdivided them into training, calibration, and test sets in the ratio $80 \%, 10 \%, 10 \%$.

\subsection{Poisson Equation}
In our second case we consider a two-dimensional Poisson problem, in which the governing equation is
\begin{equation}
-\nabla \cdot (\nabla u(x, y)) = f(x, y), \quad (x, y) \in [0,1]^2,
\end{equation}
subject to Dirichlet boundary conditions \( u(x, y) = 0 \) on \( \partial [0,1]^2 \).
The function \( f(x, y) \) represents a spatially varying source term that influences the solution \( u(x, y) \). We generate a family of such problems by sampling \( f \) from a random field with bounded support, then numerically solving for the corresponding solution \( u \). Our goal is to learn the mapping \( f(\cdot, \cdot) \mapsto u(\cdot, \cdot) \) and and again provided set-valued predictions $\Gamma_\alpha(\cdot)$.

To generate synthetic examples, we discretize the domain \([0,1]^2\) using three types of grids: (i) a standard uniform \(N \times N\) grid with evenly spaced coordinates, (ii) a non-uniform grid obtained by applying a cubic transformation to a uniform reference grid to cluster points toward the center, and (iii) a grid generated via a sine transformation to concentrate points near the domain boundaries. These coordinate mappings introduce structured resolution variation while preserving geometric continuity for the finite difference solver. We construct random forcing fields \(f\) by summing a small number of random Fourier modes. For each forcing sample, we solve the equation \(\Delta u = f\) with boundary condition \(u = 0\) using either a finite-difference or Jacobi iterative method, yielding the ground-truth solution \(u\). This process yields a dataset \(\{(f_i, u_i)\}\), which we partition into 5{,}000 training, 1{,}000 calibration, and 1{,}000 test instances.

\subsection{Unsteady Navier--Stokes Dynamics}
Our final case study considers the two-dimensional incompressible Navier--Stokes equations in vorticity form:
\begin{equation}
\partial_t \omega + \mathbf{u} \cdot \nabla \omega = \nu \Delta \omega + f, \quad (x, y) \in [0,1]^2, \ t \in [0, T],
\end{equation}
where \(\omega(x, y, t)\) denotes the vorticity, \(\nu > 0\) is the kinematic viscosity, and \(f(x, y)\) is a stationary external forcing term. The velocity field \(\mathbf{u}\) is recovered from \(\omega\) via the stream function \(\psi\), solving \(\Delta \psi = \omega\), followed by \(\mathbf{u} = (-\partial_y \psi, \partial_x \psi)\). We impose periodic boundary conditions in both spatial dimensions.

We generate the initial vorticity field \(\omega_0(x, y)\) by sampling from a two-dimensional Gaussian random field with spectral decay. The external forcing \(f(x, y)\) is chosen as a fixed sinusoidal function. We simulate the evolution of \(\omega(x, y, t)\) using a pseudo-spectral method with dealiasing and implicit treatment of the diffusion term. The solver records the vorticity field at uniformly spaced time steps, producing a spatio-temporal trajectory \(\omega(\cdot, \cdot, t)\) over a time horizon \(T = 50\) with \(200\) snapshots. We then generate 1,200 initial conditions and their corresponding time evolutions on a \(64 \times 64\) spatial grid. Next we split this into 1{,}000 training, 100 calibration, and 100 test instances. Because we don't provide a qualitative evaluation of this case in the paper, we refer the reader to Figure~\ref{fig:nva} for a visualization. 

\begin{figure}[ht]
    \centering
    \includegraphics[width=0.52\linewidth]{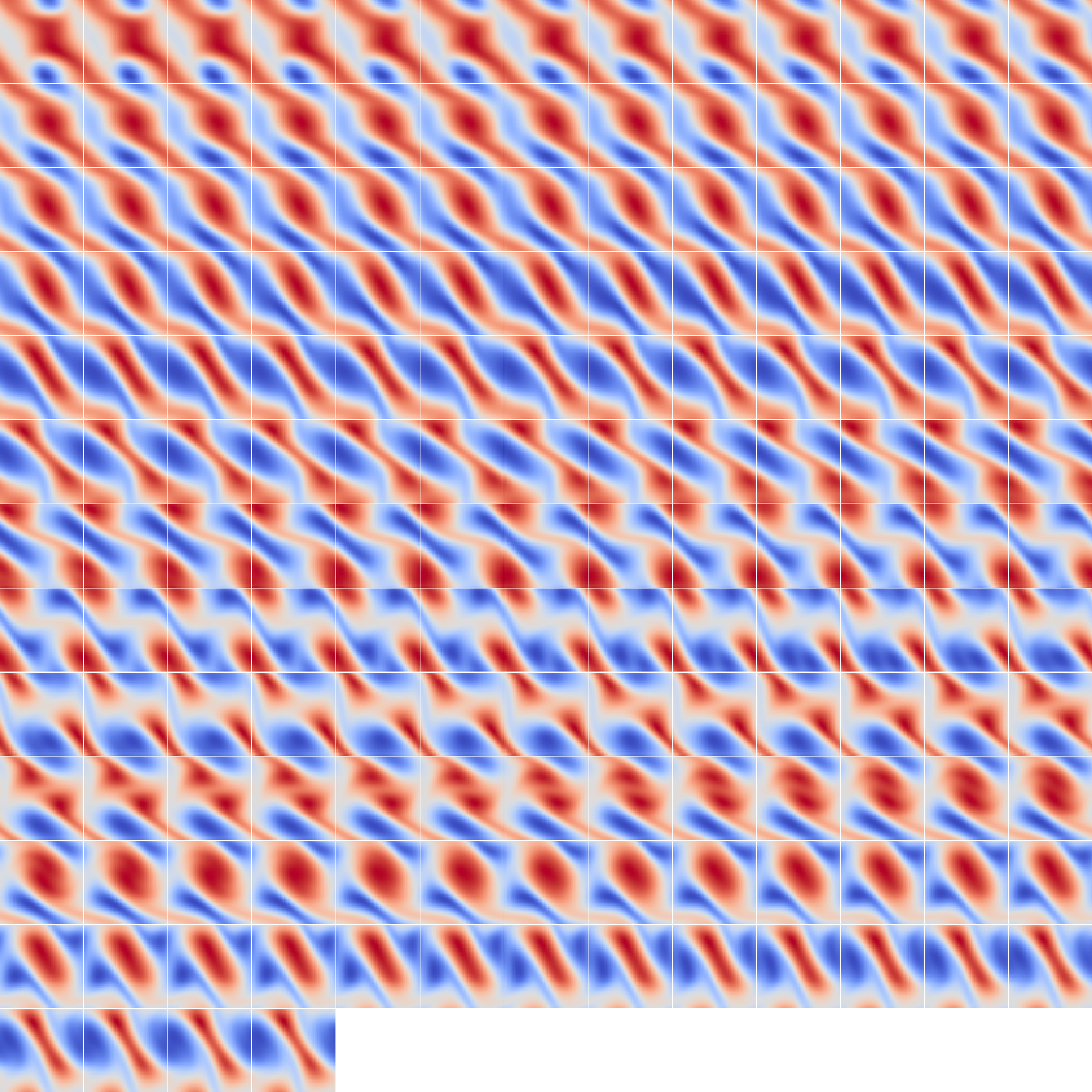}
    \caption{Unsteady Navier--Stokes Dynamics. Visualized is a single predicted trajectory from the evaluation set.}
    \label{fig:nva}
\end{figure}

\section{Discussion on Resolution Adjustment}
As shown in Figure~\ref{fig:tau_distribution}, only two dominant patterns are observed in the Poisson data—excluding the lower resolutions, which are dominated by $\varepsilon_{\text{disc}}$. In the Darcy case, a similar trend emerges, albeit with a slightly weaker spike. Thus, our regression approach proves effective for calibration transport across resolutions, especially for super-resolution tasks.
\begin{figure*}[ht]
    \centering
    \begin{minipage}[t]{0.48\textwidth}
        \centering
        \includegraphics[width=\linewidth]{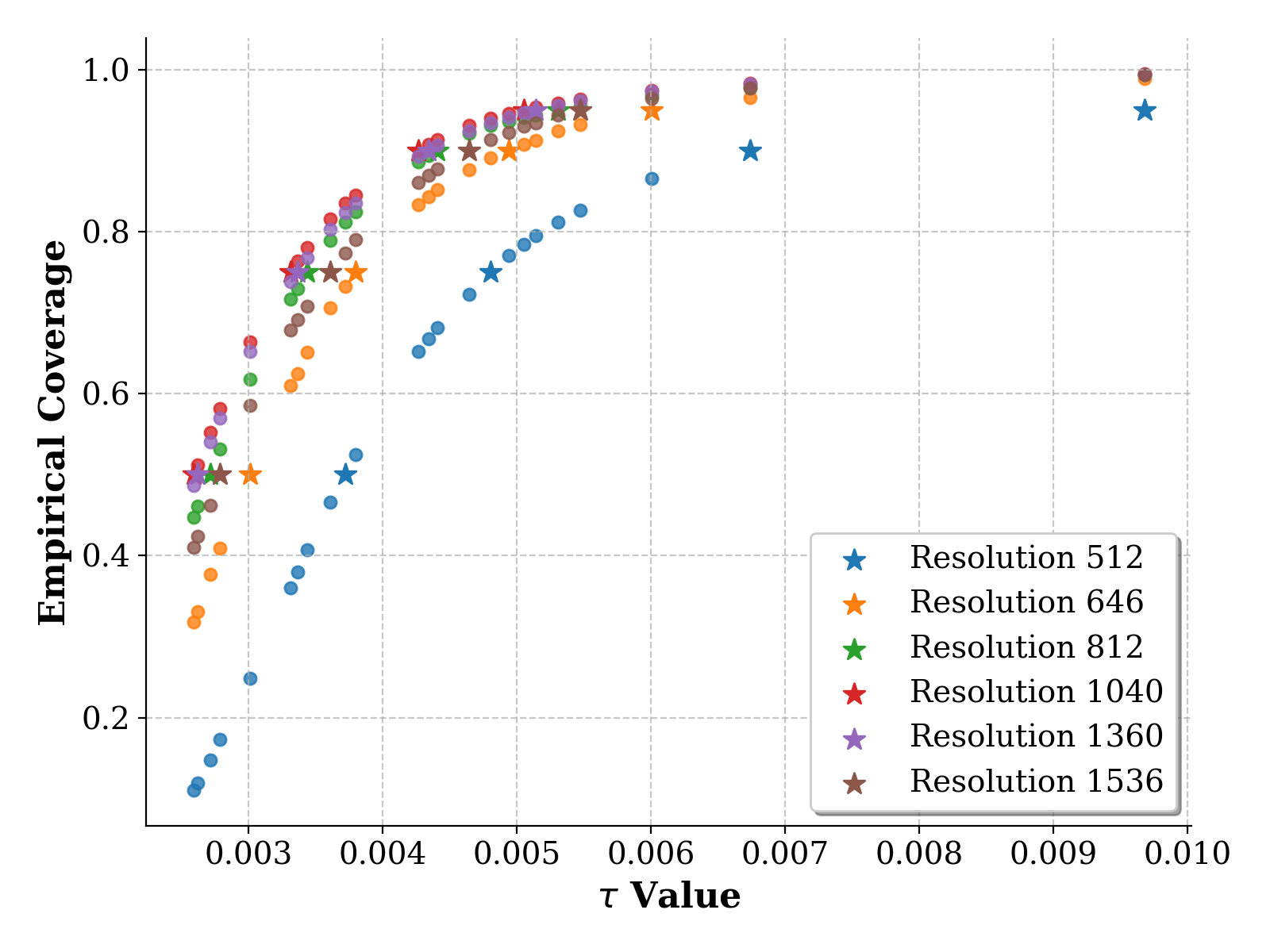}
        \caption{Coverage versus $\tau_\alpha$, grouped by resolution. We mark each resolution's own calibrated $\tau_\alpha$/coverage with a colored star; circles represent $\tau_\alpha$ values evaluated at non-native resolutions.}\label{fig:covergence_compare}
    \end{minipage}
    \hfill
    \begin{minipage}[t]{0.48\textwidth}
        \centering
        \includegraphics[width=\linewidth]{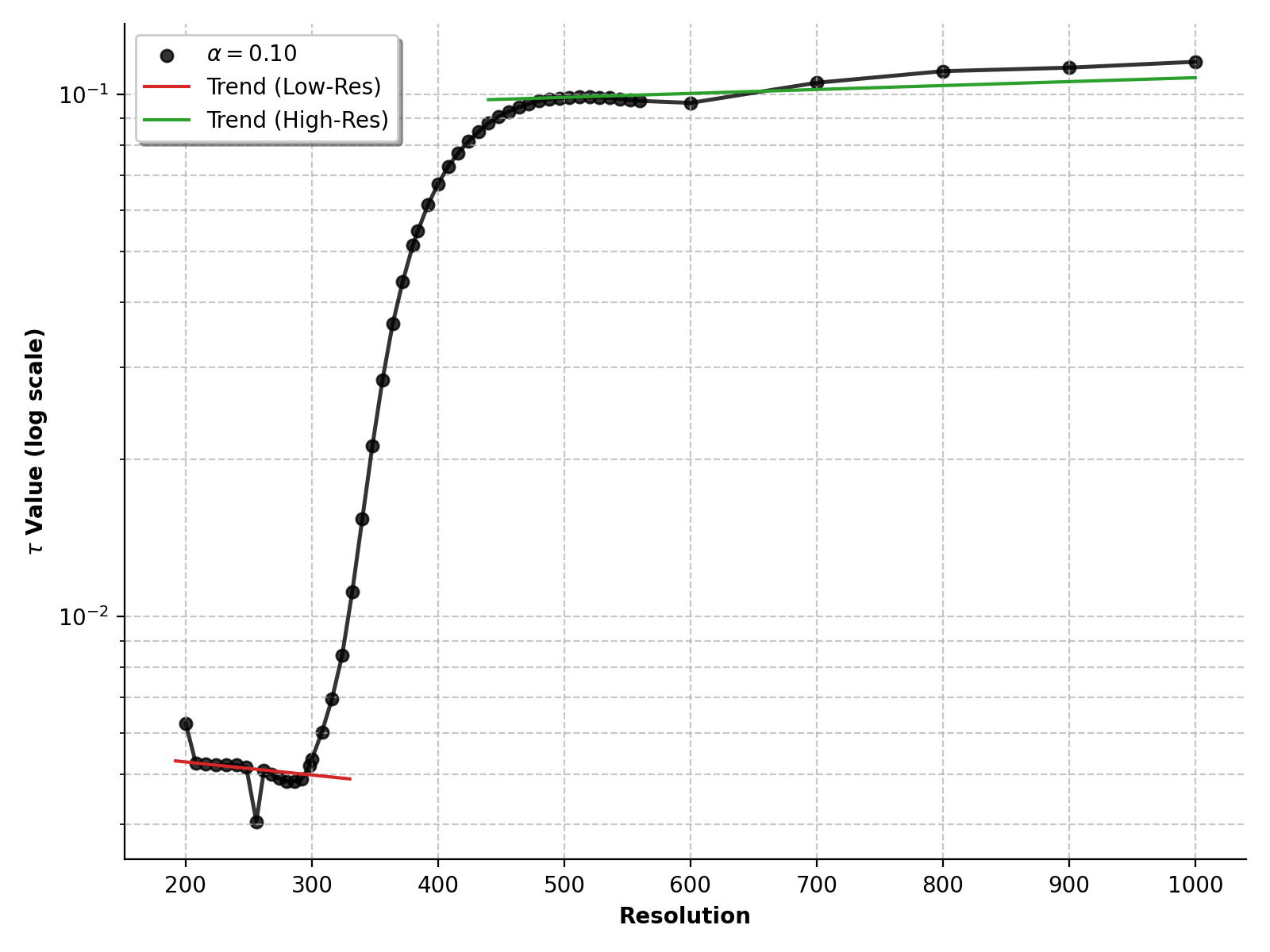}
        \caption{We plot the empirical distribution of $\tau_\alpha$ values across varying resolutions for the Poisson case study. We propose two small regressions to transport $\tau_\alpha$ within the most consistent reigns.}
        \label{fig:tau_distribution}
    \end{minipage}
\end{figure*}
In some cases, particularly for resolutions close to the calibration resolution, this transportation adjustment is not as necessary. Figure~\ref{fig:covergence_compare} shows the inter-evaluation of $\tau_\alpha$ values across a small subset of resolutions within the Darcy flow case. As we increase the resolution locally, the drop in significance—while still present—remains relatively minor compared to larger resolution shifts.  For example, consider the $\tau_\alpha$ value produced at resolution 1{,}040 for a significance level of $\alpha = 0.1$. As the resolution increases to 1{,}536, the coverage drops only slightly to 85\%. In contrast, when the resolution decreases by a similar amount to 512, the coverage drops more substantially to 65\%.s

\section{Experimental Setup}

All experiments were conducted on SPORC (Scheduled Processing On Research Computing), Rochester Institute of Technology’s High Performance Computing (HPC) cluster. For all training, evaluation, and calibration tasks, we used a single NVIDIA A100 GPU and two CPU cores. This configuration was sufficient for all neural operator models and associated conformal calibration routines. No distributed training or multi-GPU setups were required. Job submissions were managed through Slurm, and all experiments used fixed random seeds for reproducibility. We provide a summary of each methods computational budget in Table~\ref{tab:comp_bug}.
\begin{table*}[ht]
    \centering
    \caption{Computational Summary. Reported losses are dependent on training procedure (e.g., variational methods include the KL loss.) Each result is produced from an average of 5 runs.}
    \begin{tabular}{lccccc}
        \toprule
        \textbf{Model} & \textbf{Epochs} & \textbf{Dataset Size} & \textbf{\# Parameters} & \textbf{Loss} & \textbf{Training Time (d:hh:mm)} \\
        \midrule
        \multicolumn{4}{l}{\textbf{Darcy 1D}} \\
        \quad MC Dropout & 700 & 20,000 & 454K & 0.000006 & 0:00:34 \\
        \quad Variational & 700 & 20,000 & 480K & 0.000065 & 0:00:34 \\
        \midrule 
        \multicolumn{4}{l}{\textbf{Poisson 2D}} \\
        \quad Triplet & 500 & 5,000 & 19.7M & 0.005017 & 0:01:02 \\
        \quad Quantile & 500 & 5,000 & 19.7M & 0.005783 & 0:01:01 \\
        \midrule        
        \multicolumn{4}{l}{\textbf{Navier-Stokes 2D}} \\
        \quad Deterministic & 100 & 1,200 & 226.9M & 0.000004 & 3:18:28 \\
        \quad Variational & 100 & 1,200 & 227.0M & 0.000116 & 2:18:28 \\
        \bottomrule
    \end{tabular}
    \label{tab:comp_bug}
\end{table*}

\section{Super Resolution}
Figures~\ref{fig:darcy_super_ex} and~\ref{fig:poisson_super_ex} present qualitative visualizations of individual test instances evaluated far outside the training and calibration resolution. In both cases, the displayed instance corresponds to the resolution level to which the conformal calibration threshold $\tau_\alpha$ was transported using the empirical regression procedure described in Section~\ref{sec:adj}. These examples illustrate the structure of the predicted field, the associated conformal bounds, and the alignment between model uncertainty and true error, despite a significant resolution shift relative to the calibration set. The ability to apply calibrated uncertainty estimates under such resolution mismatch is critical for high-fidelity scientific applications and demonstrates the practical effectiveness of our resolution-agnostic calibration framework.
\begin{figure*}[ht]
    \centering
    \setlength{\tabcolsep}{4pt}
    \renewcommand{\arraystretch}{1.0}

    \begin{tabular}{@{}c@{}c@{}c@{}c@{}}
        % First row (MC Dropout)
        \raisebox{21mm}{\rotatebox[origin=c]{90}{\textbf{Variational}}} &
        \begin{minipage}[b]{0.31\textwidth}
            \centering
            \caption*{\hspace{4mm}(a) Prediction vs Ground Truth}
            \includegraphics[width=\textwidth]{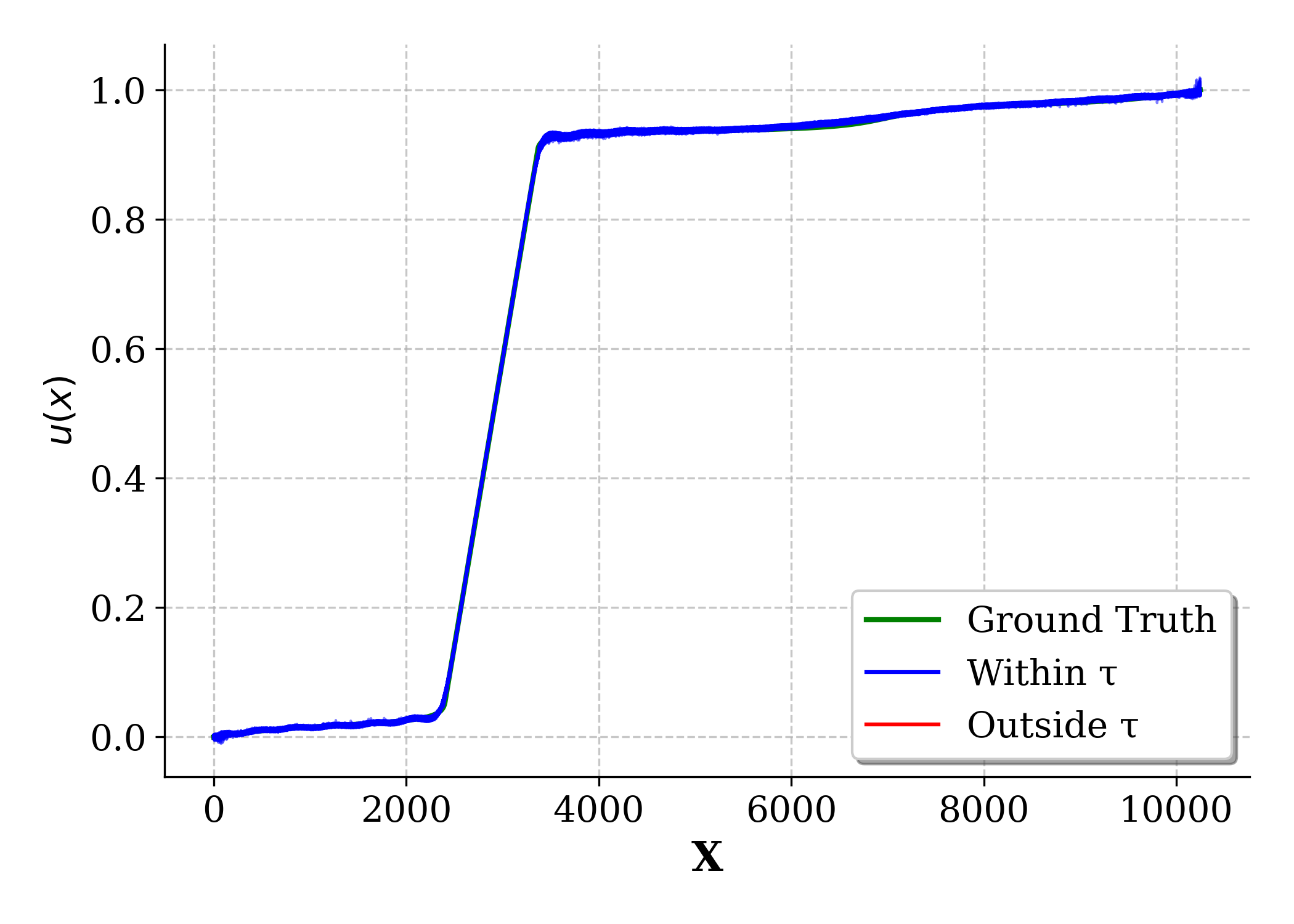}
        \end{minipage} &
        \begin{minipage}[b]{0.31\textwidth}
            \centering
            \caption*{\hspace{4mm}(b) Prediction Interval}
            \includegraphics[width=\textwidth]{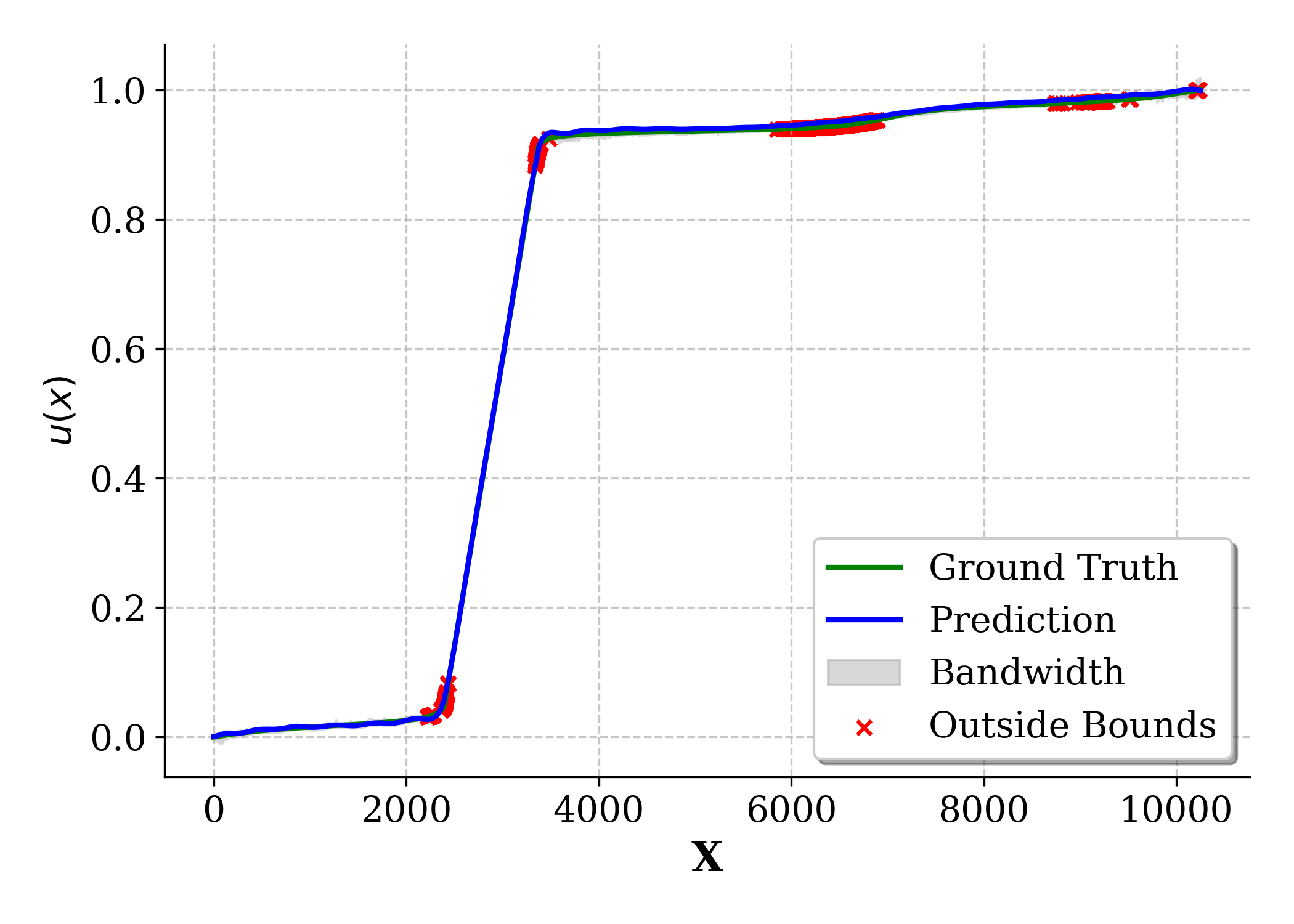}
        \end{minipage} &
        \begin{minipage}[b]{0.31\textwidth}
            \centering
            \caption*{\hspace{4mm}(c) Std. vs Prediction Error}
            \includegraphics[width=\textwidth]{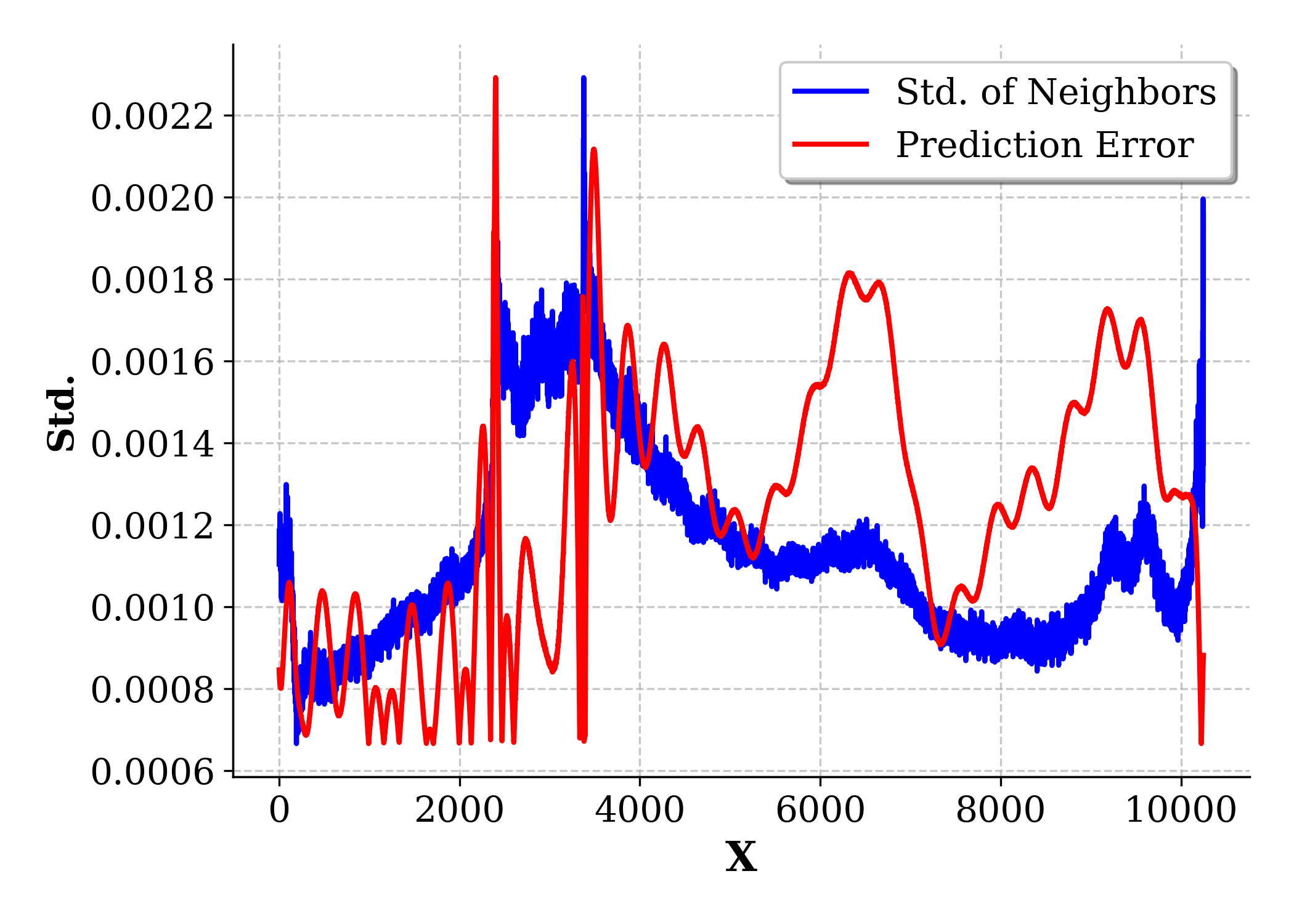}
        \end{minipage} \\[2mm]
    \end{tabular}
    \caption{Visualization of a Darcy flow test instance evaluated at a resolution significantly higher than that of the training and calibration data. This specific instance corresponds to the resolution level used for transporting the conformal calibration scalar $\tau_\alpha$ via the regression procedure. Subfigures show (a) the predicted solution alongside the ground truth, (b) the conformal prediction interval derived from sampled trajectories, and (c) the local standard deviation compared to the pointwise prediction error.}
    \label{fig:darcy_super_ex}
\end{figure*}
\begin{figure*}[ht]
    \centering
    \begin{minipage}[b]{0.32\textwidth}
        \centering
        \includegraphics[width=0.9\textwidth]{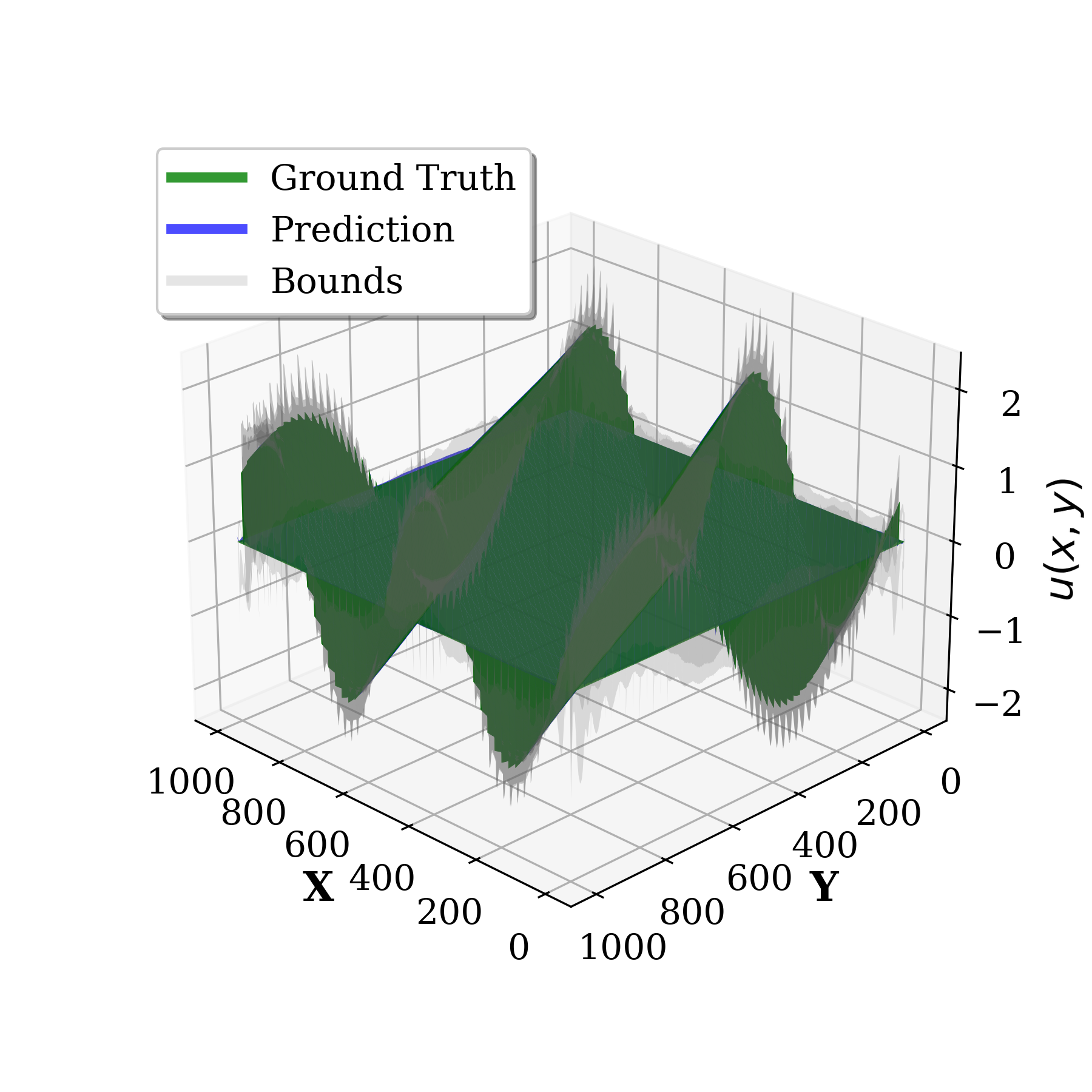}
        \caption*{\hspace{2mm}(a) Prediction vs Ground Truth}
    \end{minipage}
    \hfill
    \begin{minipage}[b]{0.32\textwidth}
        \centering
        \includegraphics[width=0.9\textwidth]{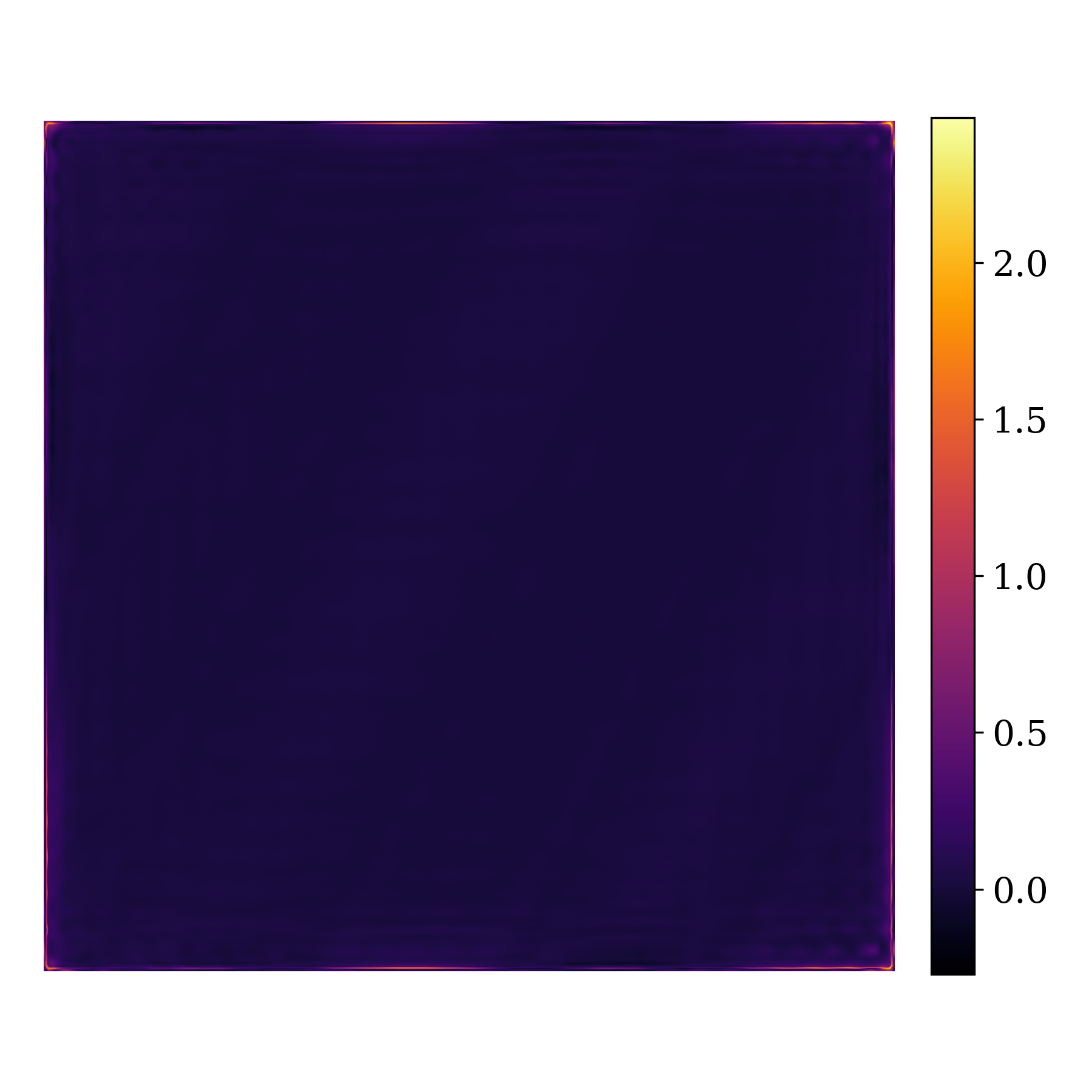}
        \caption*{\hspace{-6mm}(b) Std. of Neighbors}
    \end{minipage}
    \hfill
    \begin{minipage}[b]{0.32\textwidth}
        \centering
        \includegraphics[width=0.9\textwidth]{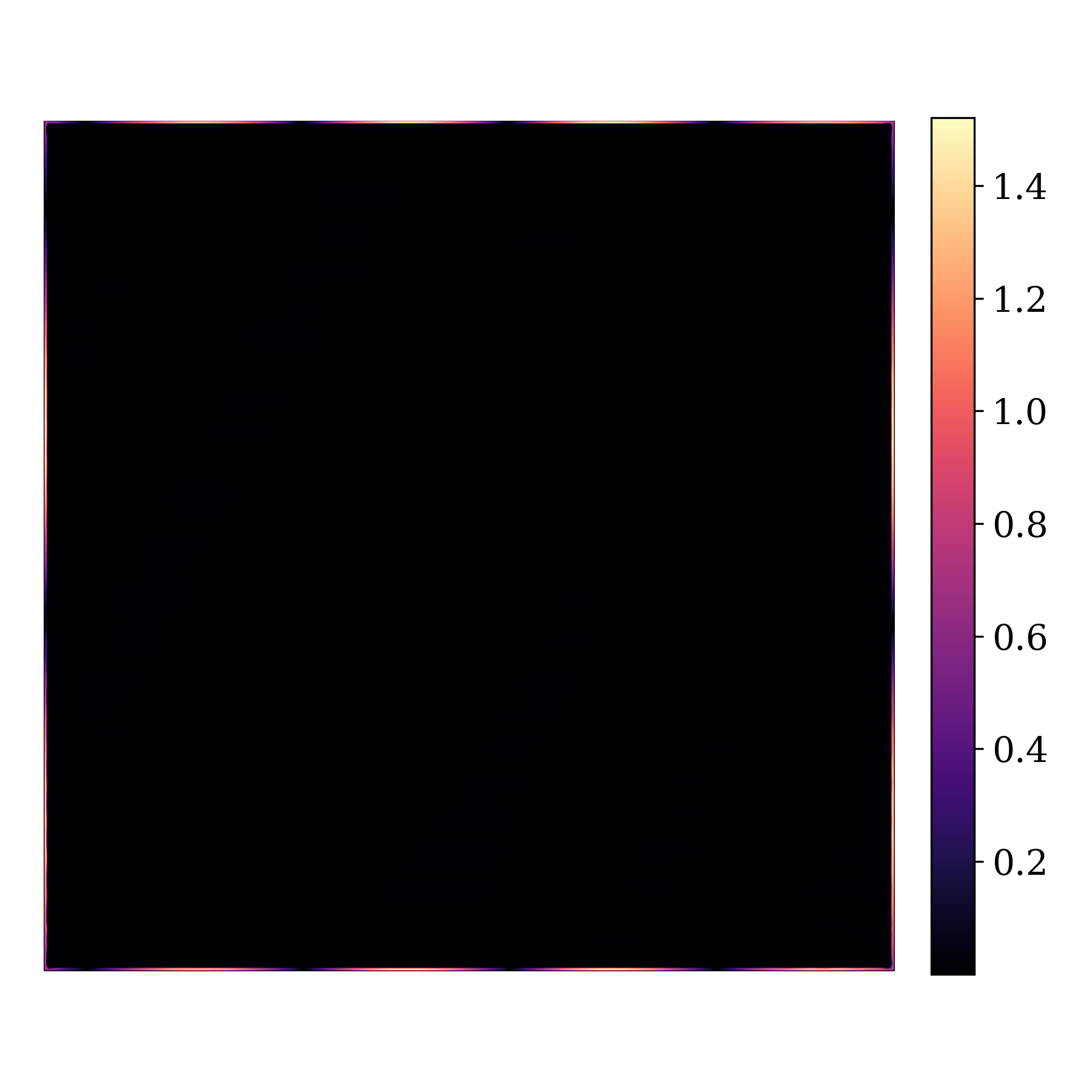}
        \caption*{\hspace{-6mm}(c) Prediction Error}
    \end{minipage}
    \caption{Visualization of a Poisson equation test instance at a super-resolved grid. This instance was used as the evaluation target for applying the transported conformal threshold $\tau_\alpha$. Shown are (a) the predicted solution versus the ground truth, (b) the local ensemble spread, and (c) the absolute prediction error. These visualizations demonstrate the model’s behavior under resolution shift beyond the training range.}
    \label{fig:poisson_super_ex}
\end{figure*}

\section{Grid Geometries}
Let $[0,1]^2$ denote the spatial domain, and let $n_x, n_y$ denote the number of discretization points along each axis. We construct grid coordinates by first defining a parametric variable $t \in [-1,1]$ sampled uniformly with $n_x$ and $n_y$ points in the $x$- and $y$-directions, respectively. The final grid coordinates $(x_i, y_j)$ are obtained via one of the following continuous mappings $T : [-1,1] \to [0,1]$ applied elementwise:
\begin{align}
\text{Uniform:} \quad & T_{\text{uniform}}(t) = \tfrac{1}{2}(t + 1), \\
\text{Clustered Center:} \quad & T_{\text{center}}(t) = \tfrac{1}{2}(t^3 + 1), \\
\text{Clustered Boundary:} \quad & T_{\text{boundary}}(t) = \tfrac{1}{2}(\sin(\tfrac{\pi}{2} t) + 1).
\end{align}
These mappings define the spatial concentration of grid points. The uniform mapping yields evenly spaced coordinates across the domain. The cubic mapping clusters points toward the domain center, due to its vanishing first derivative at $t = \pm 1$ and maximal slope at $t = 0$. Conversely, the sinusoidal mapping clusters points near the boundaries, as its derivative vanishes at $t = 0$ and grows toward $t = \pm 1$. In each case, the transformation ensures that grid coordinates remain in $[0,1]$ and preserve the geometric continuity of the domain. These geometries enable controlled evaluation of calibration robustness under structured discretization shift. We visualize each method in Figure~\ref{fig:grid_schemes}.

\begin{figure*}[ht]
    \centering
    \setlength{\tabcolsep}{4pt}
    \renewcommand{\arraystretch}{1.0}

    \begin{tabular}{@{}c@{}c@{}c@{}c@{}}
        % First row
        \raisebox{22mm}{\rotatebox[origin=c]{90}{\textbf{Uniform}}} &
        \begin{minipage}[b]{0.27\textwidth}
            \centering
            \caption*{(a) Weight Map}
            \includegraphics[width=\textwidth]{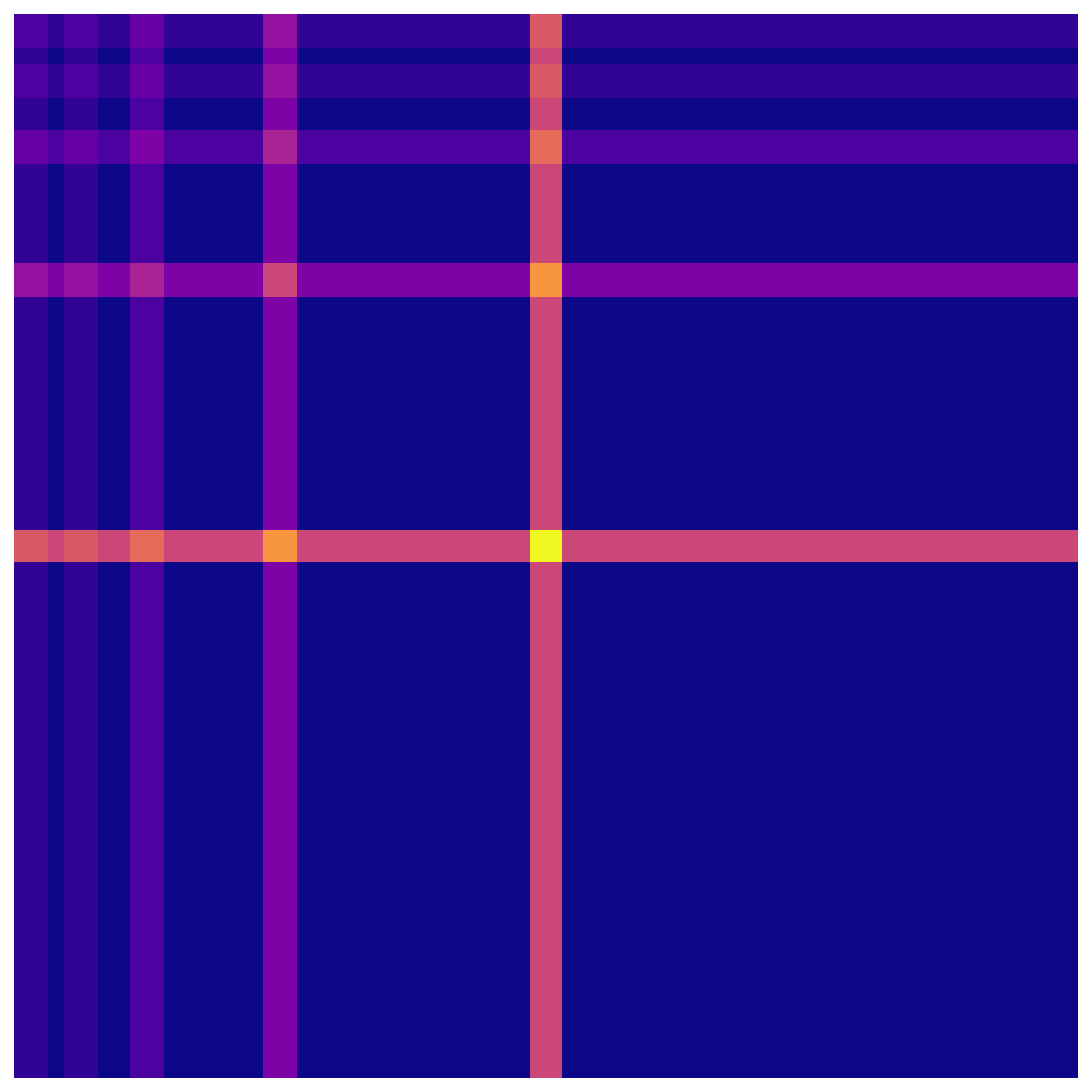}
        \end{minipage} &
        \begin{minipage}[b]{0.27\textwidth}
            \centering
            \caption*{(b) Horizontal Mesh}
            \includegraphics[width=\textwidth]{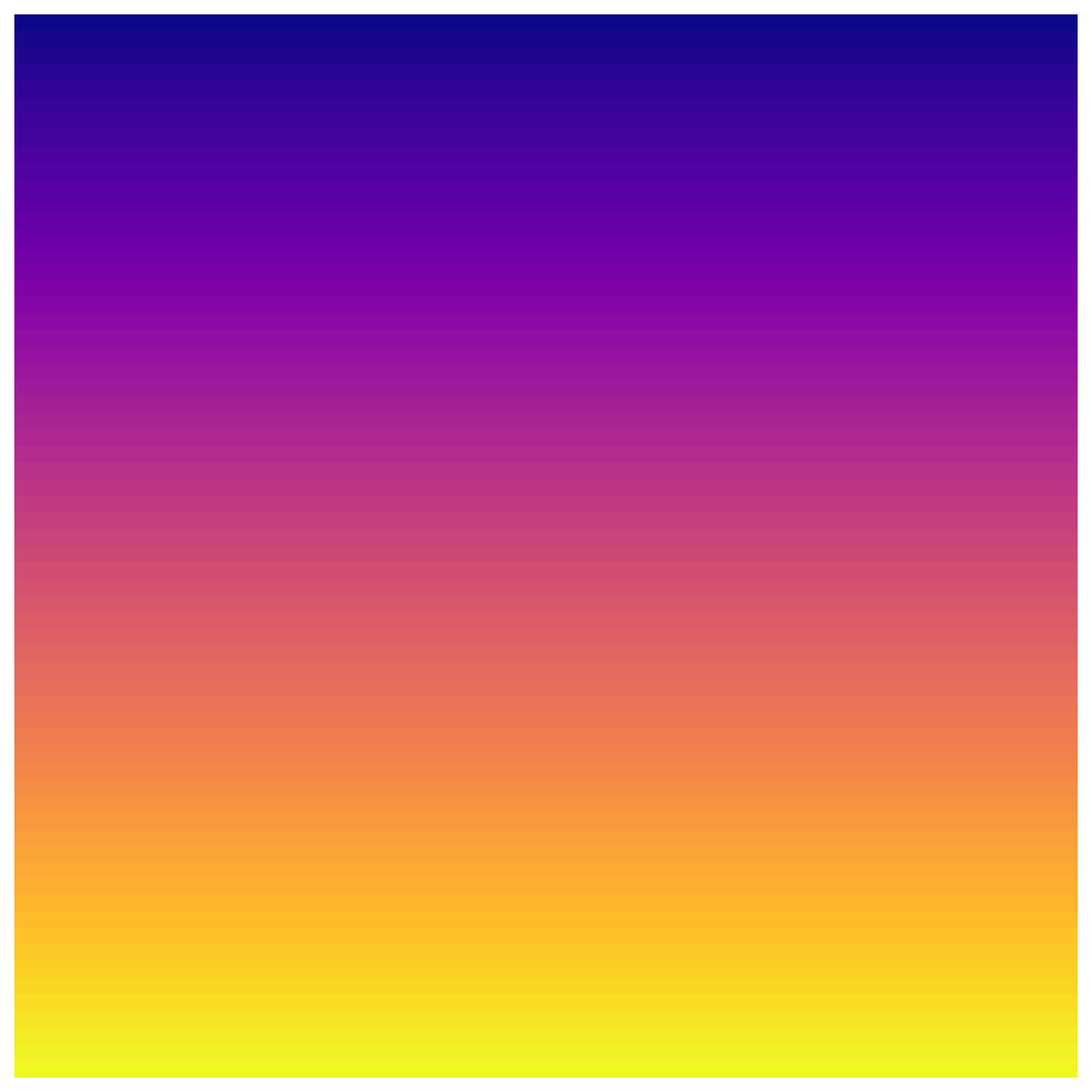}
        \end{minipage} &
        \begin{minipage}[b]{0.27\textwidth}
            \centering
            \caption*{(c) Vertical Mesh}
            \includegraphics[width=\textwidth]{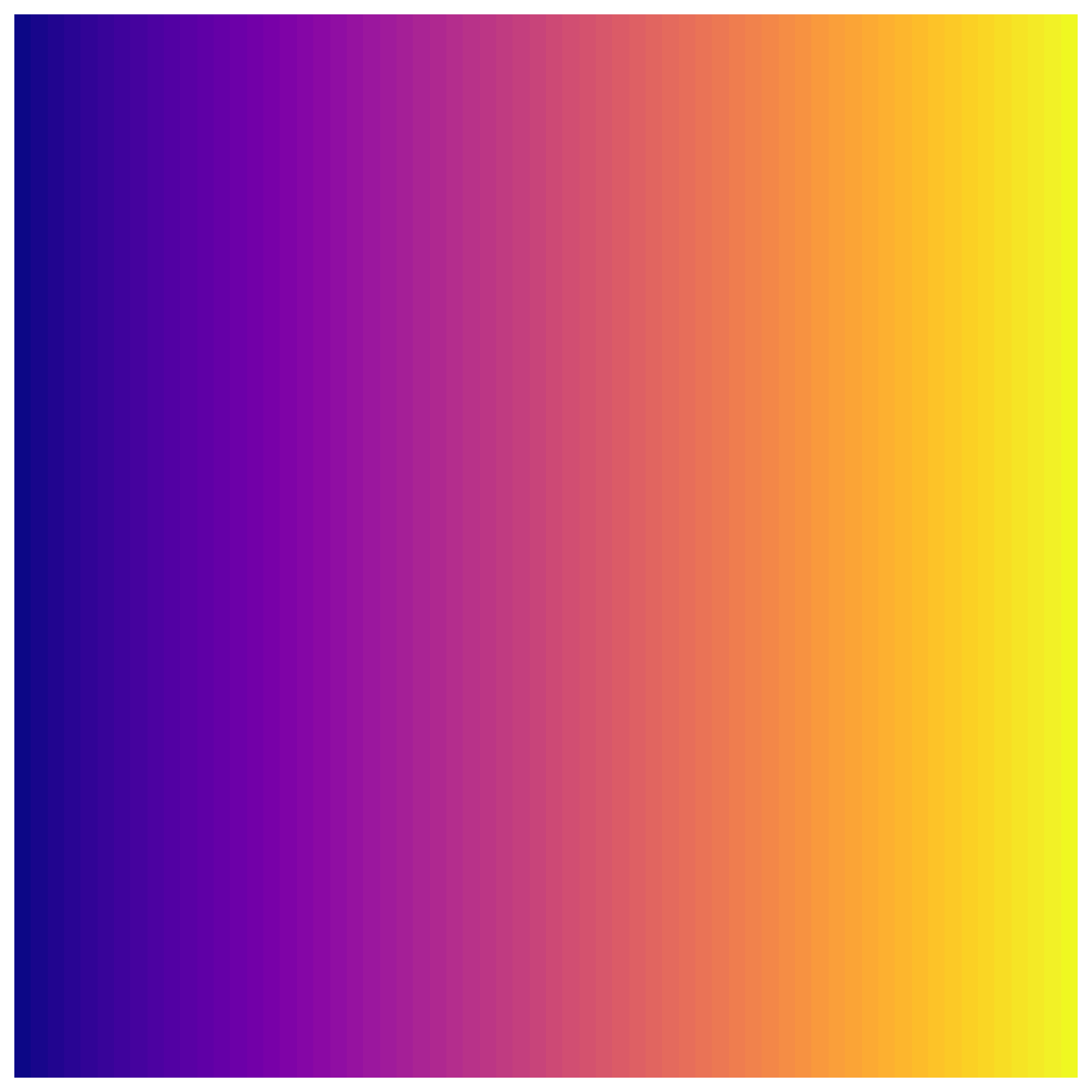}
        \end{minipage} \\[1mm]

        % Second row
        \raisebox{22mm}{\rotatebox[origin=c]{90}{\textbf{Clustered Center}}} &
        \begin{minipage}[b]{0.27\textwidth}
            \centering
            \includegraphics[width=\textwidth]{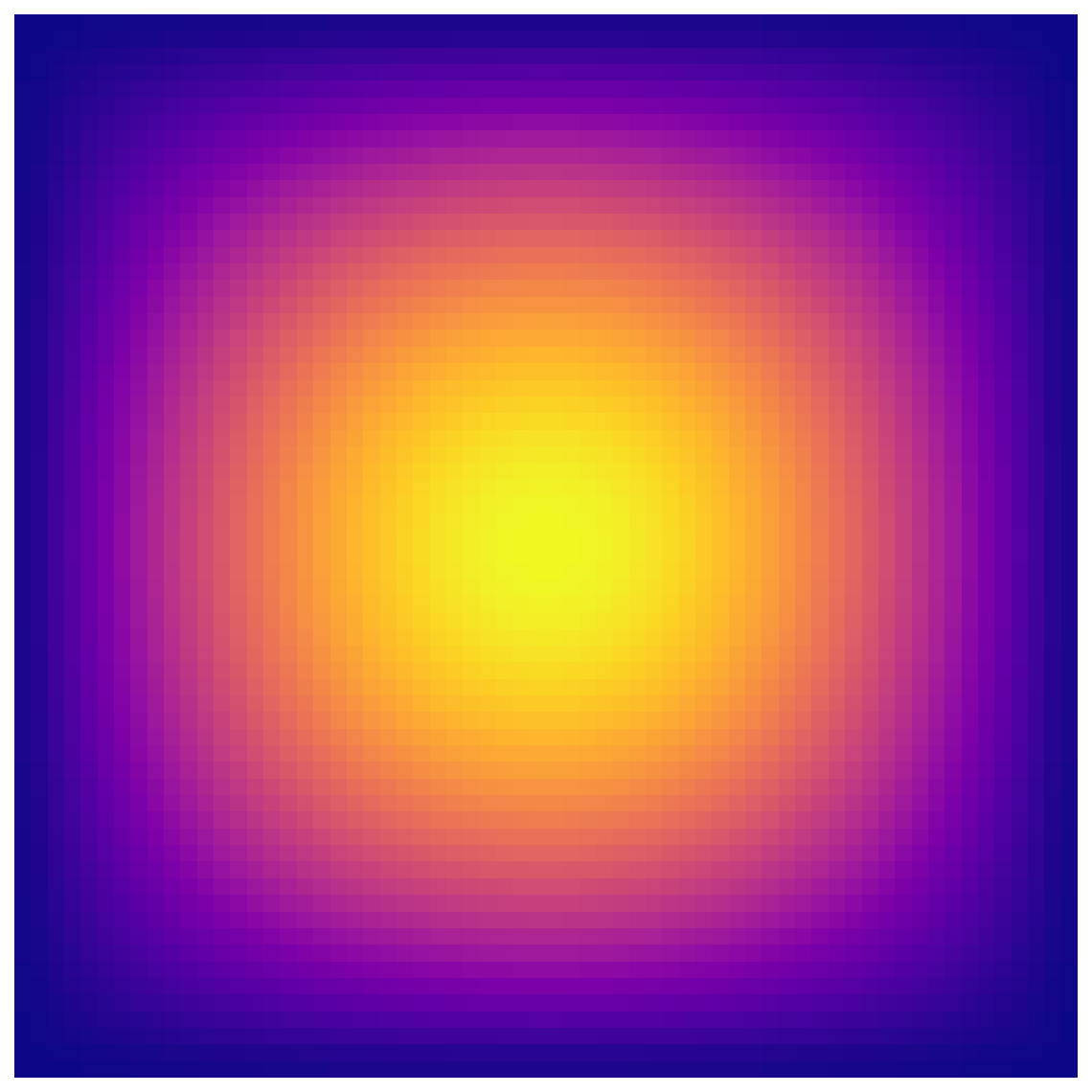}
        \end{minipage} &
        \begin{minipage}[b]{0.27\textwidth}
            \centering
            \includegraphics[width=\textwidth]{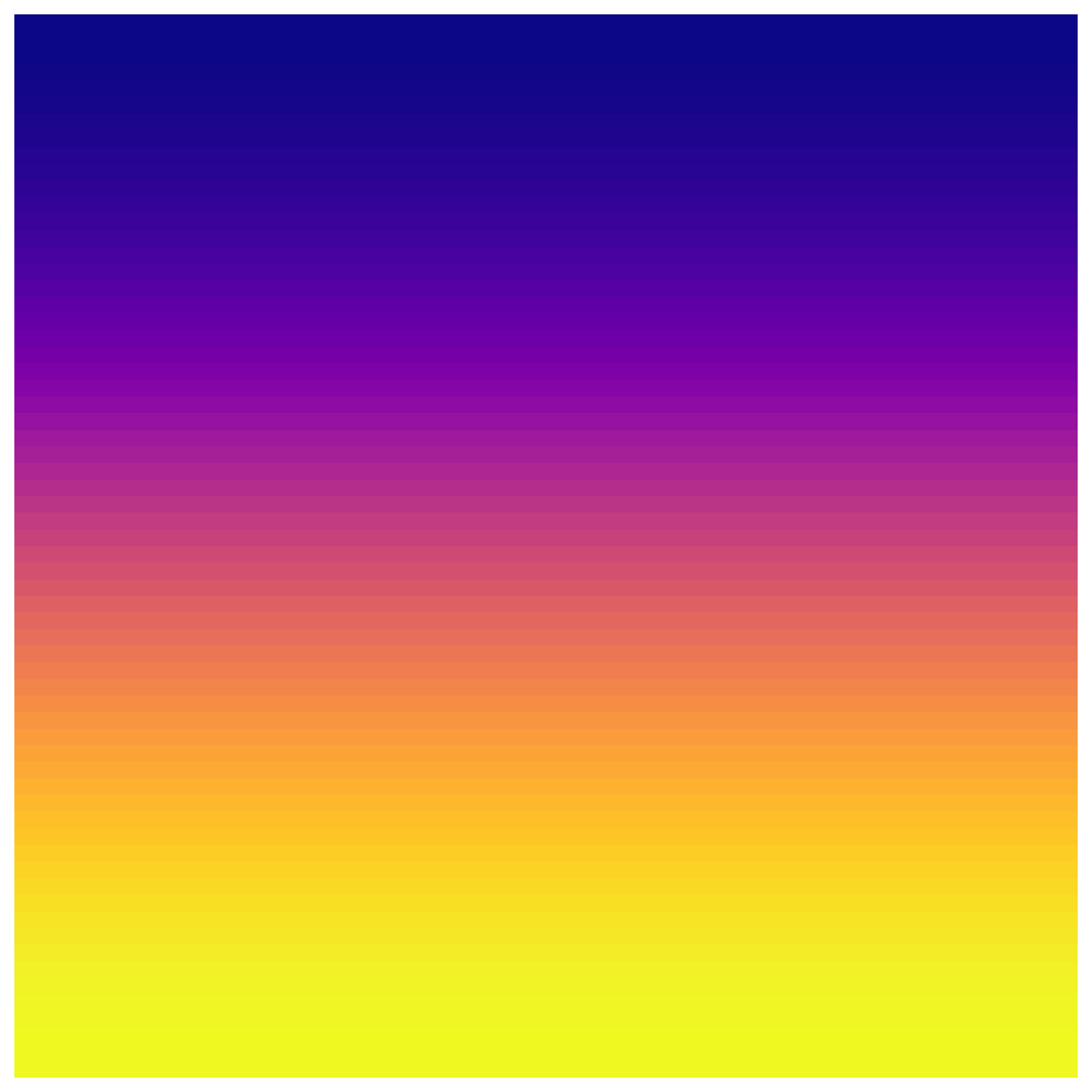}
        \end{minipage} &
        \begin{minipage}[b]{0.27\textwidth}
            \centering
            \includegraphics[width=\textwidth]{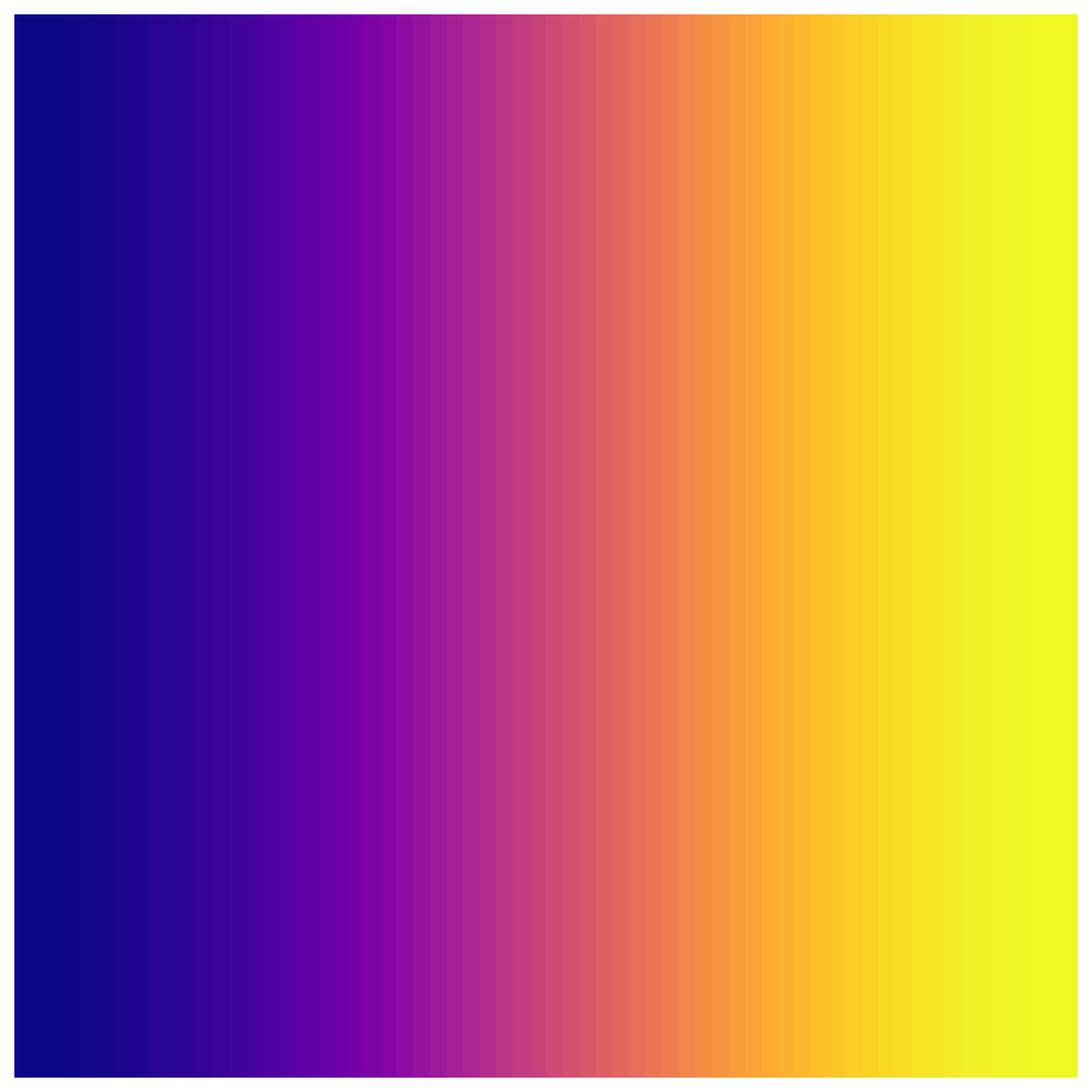}
        \end{minipage} \\[1mm]

        % Third row
        \raisebox{22mm}{\rotatebox[origin=c]{90}{\textbf{Clustered Boundary}}} &
        \begin{minipage}[b]{0.27\textwidth}
            \centering
            \includegraphics[width=\textwidth]{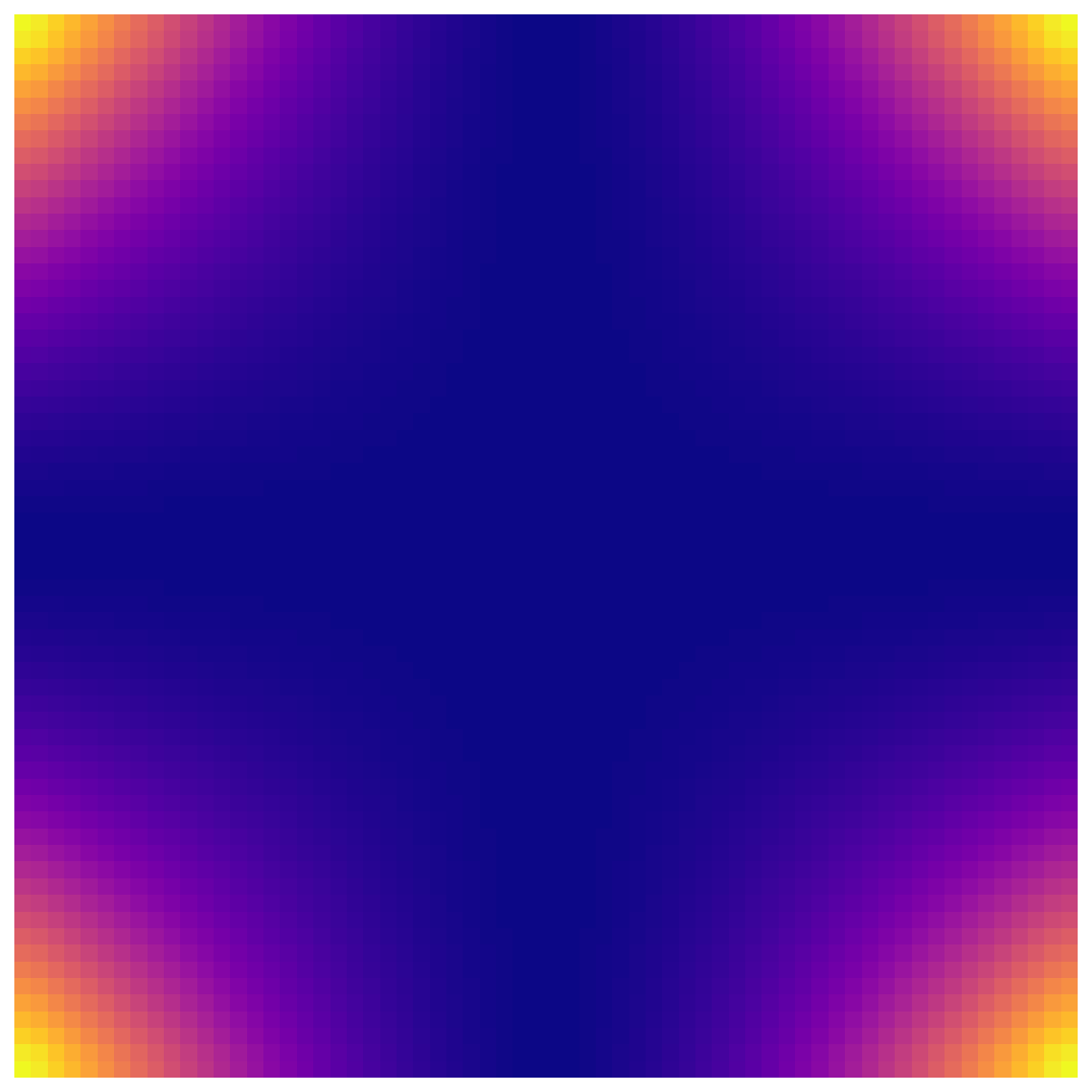}
        \end{minipage} &
        \begin{minipage}[b]{0.27\textwidth}
            \centering
            \includegraphics[width=\textwidth]{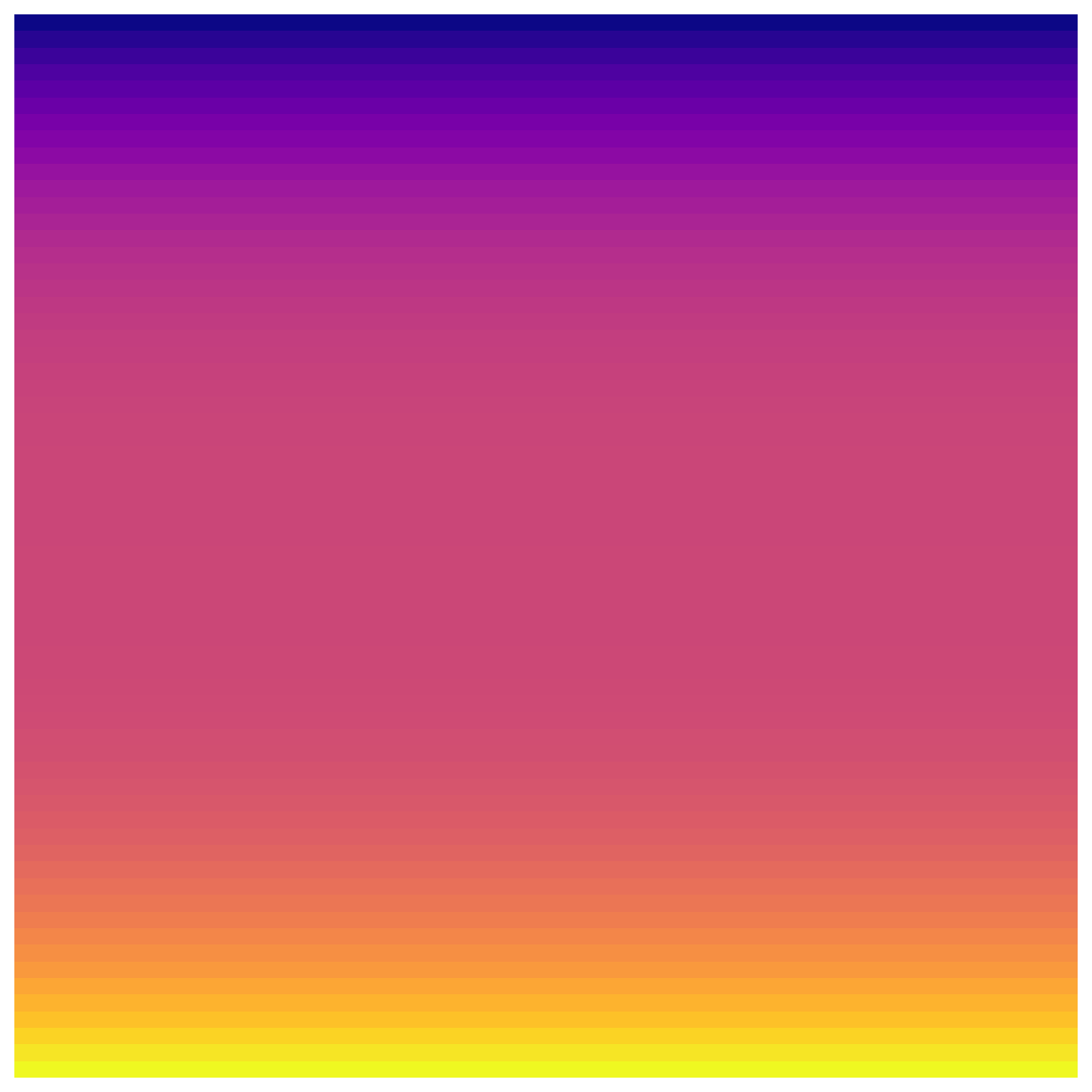}
        \end{minipage} &
        \begin{minipage}[b]{0.27\textwidth}
            \centering
            \includegraphics[width=\textwidth]{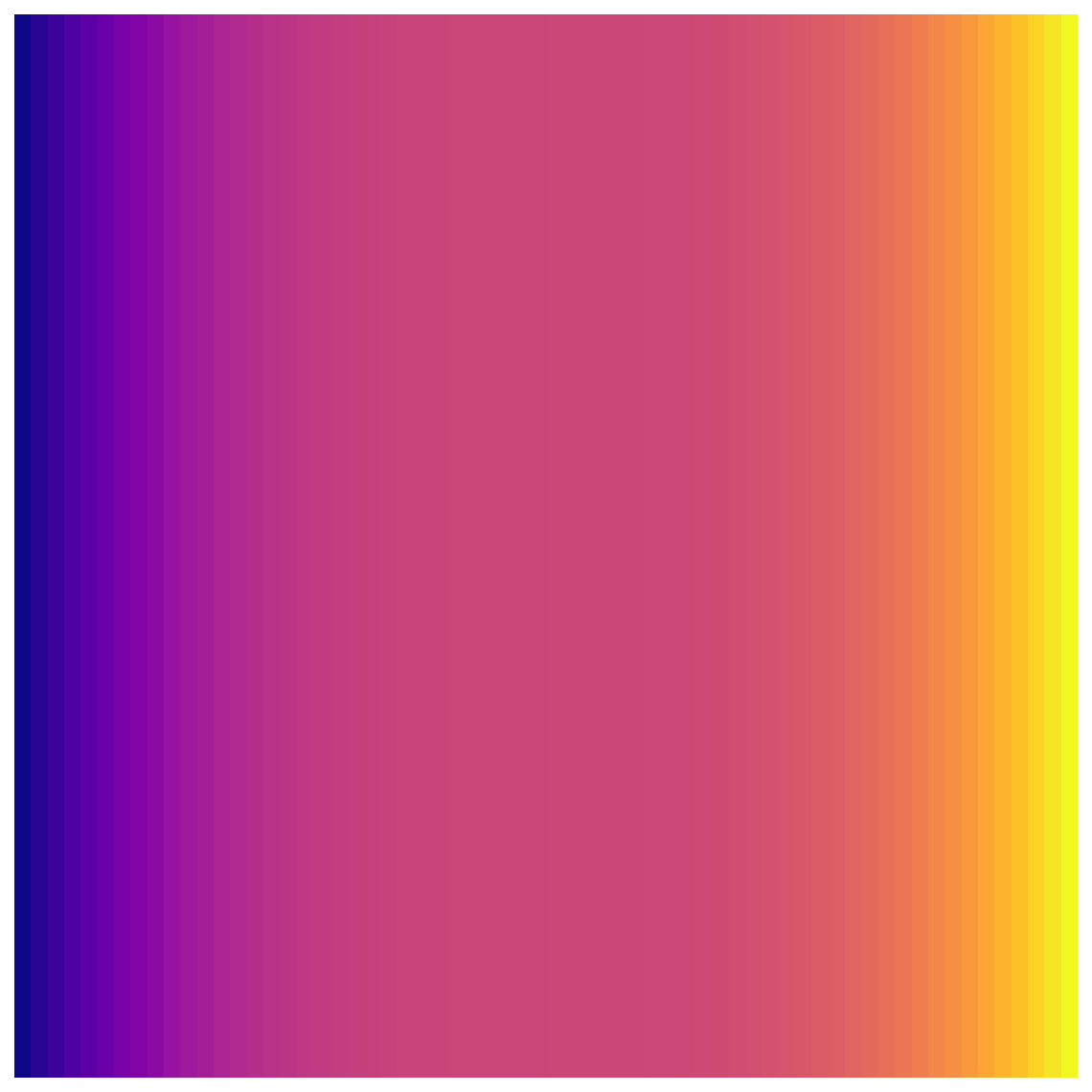}
        \end{minipage}
    \end{tabular}
    \caption{Visualization of the three grid geometries used in our experiments. Each row corresponds to a different grid mapping: uniform (top), clustered toward the center (middle), and clustered toward the boundary (bottom). Columns show the numerical quadrature weight map (left), horizontal coordinate mesh $x$ (center), and vertical coordinate mesh $y$ (right). The visible banding in the uniform case is a result of floating-point rounding errors and does not meaningfully affect the calibration procedure.}
    \label{fig:grid_schemes}
\end{figure*}

\section{Volume as a Nonconformity Metric}
The prediction set, $\Gamma_{\alpha}$, is defined as the set of all candidate functions, $v$, whose nonconformity score, $s(\hat{y}, v)$, is less than or equal to the calibrated threshold, $\tau_{\alpha}$.
\begin{equation}
\Gamma_{\alpha}(\hat{y}) = \{ v \in \mathcal{Y} \,|\, s(\hat{y}, v) \le \tau_{\alpha} \}
\end{equation}
When projected into the $d$-dimensional discretized space, the geometry of this set is determined by the norm used to calculate the volume metric.

\subsubsection*{The Ellipsoidal Prediction Set}
When using the weighted $L^2$ volume, the non-uniform weights, $w_k$, corresponding to the area of each grid cell, define a prediction set in the shape of a d-dimensional ellipsoid. The semi-axis, $a_k$, for each dimension is inversely proportional to the square root of its corresponding weight:
\begin{equation}
a_k = \frac{\tau_{\alpha}}{\sqrt{w_k}}
\end{equation}
The volume of this ellipsoid is given by:
\begin{equation}
V_{\text{ellipsoid}} = \frac{\pi^{d/2}}{\Gamma\left(\frac{d}{2} + 1\right)} \times \prod_{k=1}^{d} a_k = \frac{\pi^{d/2}}{\Gamma\left(\frac{d}{2} + 1\right)} \times \frac{(\tau_{\alpha})^d}{\sqrt{\prod_{k=1}^{d} w_k}}
\end{equation}

\subsubsection*{The Hyperspherical Prediction Set}
When using a standard (unweighted) $L^2$ volume, all weights are implicitly uniform ($w_k=1$). This simplifies the prediction set to a perfect d-dimensional hypersphere with a constant radius of $\tau_{\alpha}$. The volume formula simplifies accordingly:
\begin{equation}
V_{\text{sphere}} = \frac{\pi^{d/2}}{\Gamma\left(\frac{d}{2} + 1\right)} (\tau_{\alpha})^d
\end{equation}

\subsubsection*{Numerically Stable Log-Volume Score}
Due to the high dimensionality, direct computation of the volume can be numerically unstable. We therefore compute the negative log-volume, which yields a strictly positive and stable nonconformity score. A larger score corresponds to a smaller volume, indicating a "sharper" and more confident prediction. For the general ellipsoidal case, the score is:
\begin{equation}
s_{\text{volume}} = -\log(V) = \frac{1}{2}\sum_{k=1}^{d}\log(w_k) - d\log(\tau_{\alpha}) + \log\left(\Gamma\left(\frac{d}{2}+1\right)\right) - \frac{d}{2}\log(\pi)
\end{equation}

\subsection*{Conditional Coverage Comparison}
The following table compares the functional coverage of the two nonconformity scores across the three grid geometries used for the 2D Poisson equation, evaluated at a significance level of $\alpha=0.1$.

\begin{table}[!t]
    \centering
    \caption{Comparison of Volume and Coverage for Relative $L^2$ Norm for the spheroid.}
    \begin{tabular}{lcc}
        \toprule
        \textbf{Grid Geometry} & \textbf{Volume Score} & \textbf{Empirical Coverage} \\
        \midrule
        Uniform & 32117.191 & 0.902 \\
        Clustered Center & 30235.637 & 0.903 \\
        Clustered Boundary & 30415.73 & 0.893 \\
        \bottomrule
    \end{tabular}
    \label{tab:relative_norm_comparison}
\end{table}

\begin{table}[!th]
    \centering
    \caption{Comparison of Volume and Coverage for Weighted $L^2$ Norm on the ellipsoid.}
    \begin{tabular}{lcc}
        \toprule
        \textbf{Grid Geometry} & \textbf{Volume Score} & \textbf{Empirical Coverage} \\
        \midrule
        Uniform & 15146.912 & 0.902 \\
        Clustered Center & 9859.41 & 0.915 \\
        Clustered Boundary & 13076.615 & 0.897 \\
        \bottomrule
    \end{tabular}
    \label{tab:weighted_norm_comparison}
\end{table}

\begin{table*}[!th]
\centering
\renewcommand{\arraystretch}{1.0}
\setlength{\tabcolsep}{4pt}
\caption{Comparison of calibrated thresholds (\(\tau\)) across grid geometries at \(\alpha = 0.1\).}
\begin{minipage}{0.58\textwidth}
\centering
\textbf{(a) Volume}
\vspace{0.5em}
\begin{tabular}{lcccc}
\toprule
\textbf{Grid Type} & \textbf{R-Sphere} & \textbf{W-Ellipse} & \textbf{W-Sphere} & \textbf{R-Ellipse} \\
\midrule
Uniform & 32,117.19 & 15,146.41 & 32,117.19 & 15,146.91 \\
Clustered Center & 30,235.64 & 9,859.41 & 30,162.98 & 9,932.07 \\
Clustered Boundary & 30,415.73 & 13,076.62 & 31,156.65 & 12,335.70 \\
\midrule
\multicolumn{5}{l}{\textbf{Variation}} \\
Std. & 1,038.2 & 2,664.2 & \textbf{977.2} & 2,612.3 \\
Coeff. of Var. & 3.4\% & 21.0\% & \textbf{3.1\%} & 21.0\% \\
\bottomrule
\end{tabular}
\end{minipage}
\hfill
\begin{minipage}{0.38\textwidth}
\centering
\textbf{(b) Threshold}
\vspace{0.5em}
\begin{tabular}{lcc}
\toprule
\textbf{Grid Type} & \textbf{R-Norm} & \textbf{W-Norm} \\
\midrule
Uniform & 0.02810 & 0.02810 \\
Clustered Center & 0.02866 & 0.02821 \\
Clustered Boundary & 0.03280 & 0.03079 \\
\midrule
\multicolumn{3}{l}{\textbf{Variation}} \\
Std. & 0.00257 & \textbf{0.00148} \\
Coeff. of Variation & 8.6\% & \textbf{5.1\%} \\
\bottomrule
\end{tabular}
\end{minipage}

\label{tab:tau_side_by_side_app}
\end{table*}

\end{document}